\documentclass[lettersize,journal]{IEEEtran}

\usepackage{amsmath,amsfonts}
\usepackage{array}
\usepackage{textcomp}
\usepackage{stfloats}
\usepackage{url}
\usepackage{verbatim}
\usepackage{graphicx}
\usepackage{cite}

\usepackage{subfiles}

\newcommand{\notinsubfile}[1]{}
\usepackage[linesnumbered,lined,ruled]{algorithm2e}
\usepackage[export]{adjustbox}

\usepackage[T1]{fontenc}
\usepackage{amsmath} 
\usepackage{amssymb}  
\usepackage{bbm}
\usepackage{cancel}
\usepackage[hidelinks]{hyperref}
\usepackage{bookmark}

\usepackage[shortlabels]{enumitem}

\usepackage[export]{adjustbox}
\usepackage{xcolor}

\usepackage{balance}

\usepackage{amsthm}
\usepackage{graphicx}
\usepackage{subcaption}
\usepackage{multirow}

\newcommand{\T}{\widetilde{T}}
\newcommand{\N}{\widetilde{n}}
\newcommand{\U}{\widetilde{u}}
\newcommand{\V}{\widetilde{v}}
\newcommand{\A}{\widetilde{a}}

\newcommand{\R}{\mathbb{R}}
\newcommand{\norm}[1]{\left\lVert #1\right\rVert}

\DeclareMathOperator{\argmin}{arg\,min}

\usepackage{mathtools}

\newtheorem{theorem}{Theorem}
\newtheorem{lemma}{Lemma}
\newtheorem{corollary}{Corollary}
\newtheorem{proposition}{Proposition}
\theoremstyle{definition}

\newtheorem{assumption}{Assumption}
\theoremstyle{definition}
\newtheorem{definition}{Definition}
\theoremstyle{remark}
\newtheorem{remark}{Remark}

\title{\LARGE \bf
Barrier Function Overrides
For Non-Convex
Fixed Wing Flight Control
and Self-Driving Cars
}

\author{
Eric Squires \IEEEmembership{Member, IEEE},
Phillip Odom,
Zsolt Kira
\thanks{
Eric Squires and Phillip Odom are with the Georgia Tech Research Institute.
Zsolt Kira is with the College of Computing at the Georgia Institute of Technology.
Corresponding Author: Eric Squires. 250 14th St, Atlanta, GA, 30318.
(email: eric.squires@gtri.gatech.edu; phillip.odom@gtri.gatech.edu; zkira@gatech.edu)
}
}

\begin{document}

\maketitle
\thispagestyle{empty}
\pagestyle{empty}

\begin{abstract}

Reinforcement Learning (RL) has enabled vast performance improvements for
robotics systems. To achieve these results though, the agent often must
randomly explore the environment, which for safety critical systems presents a
significant challenge. Barrier functions can solve this challenge by enabling
an override that approximates the RL control input as closely as possible without
violating a safety constraint. Unfortunately, this override can be
computationally intractable in cases where the dynamics are not convex in the
control input or when time is discrete, as is often the case when training RL
systems. We therefore consider these cases, developing novel barrier
functions for two non-convex systems
(fixed wing aircraft and self-driving cars performing lane merging with adaptive cruise control)
in discrete time.
Although solving for an online and optimal override is in general intractable
when the dynamics are nonconvex in the control input,
we investigate approximate solutions, finding that these approximations enable
performance commensurate with baseline RL methods with zero safety violations.
In particular, even without attempting to solve for the optimal override at all,
performance is still competitive with baseline RL performance. We discuss the tradeoffs
of the approximate override solutions including performance and computational tractability.
\end{abstract}

\section{Introduction}

\IEEEPARstart{R}{einforcement} learning (RL) presents significant deployment
challenges for safety critical systems. While the performance can often vastly
exceed other approaches, the propensity to unpredictably fail can render its
employment impractical. On the other hand, barrier functions
\cite{ames2016control,ames2019control}
excel at enabling
safety assurance, but require simplifying assumptions on the system dynamics to
compute control inputs that enforce the safety constraint,
making it difficult to apply to problems where RL excels.
Two examples of this type of problem are waypoint following for fixed wing aircraft
and lane merging with adaptive cruise control
for self-driving cars, both of which
can be non-convex in the control input
and therefore make computing a safe control computationally intractable.
Nevertheless, after deriving novel barrier functions for these systems
we investigate computationally tractable approximations to a barrier function override,
finding that this enables both safety assurance and high performance even
relative to model free safe RL baselines.

A barrier function
is
an output function of the system state that enables safety assurance.
When used to ensure safety with RL, the RL policy first outputs
a nominal control value. Given this control value,
a safe control value is computed through an optimization
that is
as close as possible to the nominal control value without violating 
the barrier constraint.
As shown in \cite{ames2016control}, under the assumption that the system
is linear in the control input, this optimization
can be computed efficiently for online overrides.

\begin{figure}
\centering
\def\svgwidth{0.95\columnwidth}
\begingroup%
  \makeatletter%
  \providecommand\color[2][]{%
    \errmessage{(Inkscape) Color is used for the text in Inkscape, but the package 'color.sty' is not loaded}%
    \renewcommand\color[2][]{}%
  }%
  \providecommand\transparent[1]{%
    \errmessage{(Inkscape) Transparency is used (non-zero) for the text in Inkscape, but the package 'transparent.sty' is not loaded}%
    \renewcommand\transparent[1]{}%
  }%
  \providecommand\rotatebox[2]{#2}%
  \newcommand*\fsize{\dimexpr\f@size pt\relax}%
  \newcommand*\lineheight[1]{\fontsize{\fsize}{#1\fsize}\selectfont}%
  \ifx\svgwidth\undefined%
    \setlength{\unitlength}{221.71216965bp}%
    \ifx\svgscale\undefined%
      \relax%
    \else%
      \setlength{\unitlength}{\unitlength * \real{\svgscale}}%
    \fi%
  \else%
    \setlength{\unitlength}{\svgwidth}%
  \fi%
  \global\let\svgwidth\undefined%
  \global\let\svgscale\undefined%
  \makeatother%
  \begin{picture}(1,0.61421332)%
    \lineheight{1}%
    \setlength\tabcolsep{0pt}%
    \put(0,0){\includegraphics[width=\unitlength,page=1]{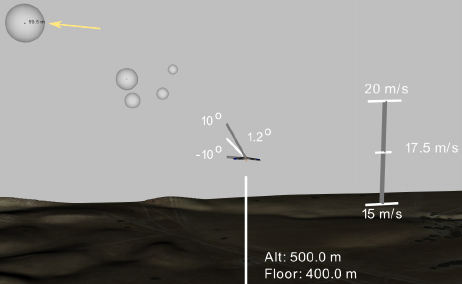}}%
    \put(0.44741809,0.5409323){\makebox(0,0)[t]{\lineheight{1.25}\smash{\begin{tabular}[t]{c}Goal: reach waypoints\end{tabular}}}}%
    \put(0,0){\includegraphics[width=\unitlength,page=2]{./imgs/svg/pdf/fw_env.pdf}}%
    \put(0.08764086,0.32545356){\makebox(0,0)[t]{\lineheight{1.25}\smash{\begin{tabular}[t]{c}Pitch\\Limits\end{tabular}}}}%
    \put(0,0){\includegraphics[width=\unitlength,page=3]{./imgs/svg/pdf/fw_env.pdf}}%
    \put(0.71859528,0.2949912){\makebox(0,0)[t]{\lineheight{0.94999999}\smash{\begin{tabular}[t]{c}Speed\\Limits\end{tabular}}}}%
    \put(0,0){\includegraphics[width=\unitlength,page=4]{./imgs/svg/pdf/fw_env.pdf}}%
    \put(0.21349509,0.19092648){\makebox(0,0)[t]{\lineheight{0.94999999}\smash{\begin{tabular}[t]{c}Altitude Limit\end{tabular}}}}%
    \put(0,0){\includegraphics[width=\unitlength,page=5]{./imgs/svg/pdf/fw_env.pdf}}%
  \end{picture}%
\endgroup%

\caption{
Screenshot of the UAV environment with safety conditions
visualized with SCRIMMAGE \cite{demarco2018}.
}
\label{fig_fw_screenshot}
\end{figure}

\begin{figure}
\centering
\def\svgwidth{0.95\columnwidth}
\begingroup%
  \makeatletter%
  \providecommand\color[2][]{%
    \errmessage{(Inkscape) Color is used for the text in Inkscape, but the package 'color.sty' is not loaded}%
    \renewcommand\color[2][]{}%
  }%
  \providecommand\transparent[1]{%
    \errmessage{(Inkscape) Transparency is used (non-zero) for the text in Inkscape, but the package 'transparent.sty' is not loaded}%
    \renewcommand\transparent[1]{}%
  }%
  \providecommand\rotatebox[2]{#2}%
  \newcommand*\fsize{\dimexpr\f@size pt\relax}%
  \newcommand*\lineheight[1]{\fontsize{\fsize}{#1\fsize}\selectfont}%
  \ifx\svgwidth\undefined%
    \setlength{\unitlength}{201.43387985bp}%
    \ifx\svgscale\undefined%
      \relax%
    \else%
      \setlength{\unitlength}{\unitlength * \real{\svgscale}}%
    \fi%
  \else%
    \setlength{\unitlength}{\svgwidth}%
  \fi%
  \global\let\svgwidth\undefined%
  \global\let\svgscale\undefined%
  \makeatother%
  \begin{picture}(1,0.5500469)%
    \lineheight{1}%
    \setlength\tabcolsep{0pt}%
    \put(0,0){\includegraphics[width=\unitlength,page=1]{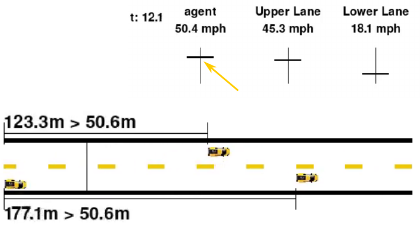}}%
    \put(0.7293396,0.29770803){\makebox(0,0)[t]{\lineheight{1.25}\smash{\begin{tabular}[t]{c}Goal: cruise control\end{tabular}}}}%
    \put(0,0){\includegraphics[width=\unitlength,page=2]{./imgs/svg/pdf/car_env.pdf}}%
    \put(0.18199136,0.45173435){\makebox(0,0)[t]{\lineheight{1.25}\smash{\begin{tabular}[t]{c}No speeding\end{tabular}}}}%
    \put(0.62413725,0.00934244){\makebox(0,0)[t]{\lineheight{1.25}\smash{\begin{tabular}[t]{c}Safe distance\end{tabular}}}}%
    \put(0,0){\includegraphics[width=\unitlength,page=3]{./imgs/svg/pdf/car_env.pdf}}%
    \put(0.17734045,0.34444005){\makebox(0,0)[t]{\lineheight{1.25}\smash{\begin{tabular}[t]{c}Lane keeping\end{tabular}}}}%
  \end{picture}%
\endgroup%

\caption{
Screenshot of the car environment with safety conditions.
}
\label{fig_car_screenshot}
\end{figure}

However, restricting allowable dynamics
to those that are linear in the control input may limit the fidelity of system.
For instance, air drag is nonlinear in the load factor for aircraft.
Similarly,
due to transformations between reference frames, trigonometric
functions give rise to systems that are non-convex in the control input.
A more thorough list of
such systems can be found in \cite{binazadeh2017robust,narang2012analysis}.
At the same time, RL does not make this assumption,
meaning that finding a way for barrier functions
to be successfuly employed in systems that are nonlinear in the control input
can broaden their applicability for difficult, safety critical problems.

Aside from being applicable to dynamics that are nonconvex in the control input,
RL introduces an additional challenge because of its discrete time
formulation.
As shown in \cite{agrawal2017discrete}, when time is discrete,
the optimization to compute a safe control
input can become non-convex,
which cannot in general be solved in real time. This can occur even when the dynamics
are linear.
One solution is to accept that there is a gap between the theory
and implementation by deriving barrier function properties in continuous
time and then deploying them in discrete time (e.g. \cite{rabiee2023safe,cheng2019end}).
Unfortunately,
this requires the discrete timestep to be small so that the continuous time
dynamics
and barrier function derivative 
are reasonable approximations of the discrete step updates.
As an example from \cite{squires2021safety},
the simulation had to run at 200 Hz to avoid collisions.
Short timesteps make it more difficult
for RL algorithms to learn \cite{pateria2021hierarchical}
so this can significantly increase training time or impede final performance.
Another approach \cite{zhang2022control}
accounts for the discretization errors by including
an adjustment to the barrier constraint condition
that is the worst case error over the reachable set
within one timestep. However,
computing this reachable set for nonlinear
systems requires conservative approximations to make it computationally
efficient.

Given the non-convex optimization in discrete time barrier functions, a variety of solutions
have been proposed. In \cite{zeng2021safety}, the authors use model predictive control (MPC)
to solve an optimization including the barrier function constraint, a Lyapunov constraint,
and performance objectives. They show that by using more than the one step horizon
the performance can be improved. However, as it is computationally prohibitive
to solve a nonlinear MPC problem as the horizon grows, the optimization may be infeasible.
Multi-step horizons have been shown elsewhere to either improve performance
\cite{yang2023model}, expand the safe set \cite{gurriet2018online},
or construct
barrier functions \cite{squires2022composition}. Recursive feasibility for nonlinear
MPC with barrier functions is analyzed in
\cite{katriniok2023discrete} which introduces quasi-barrier functions
that can be used to expand the set of feasible states and ensure recursive feasibility.
They also apply the result to lane merging although the approach
requires a centralized computation involving the other vehicle.
Aside from using 
MPC to address the non-convex optimization,
another approach \cite{cheng2020safe} is to find
a lower bound to the barrier constraint where the lower bound is linear in the control input.
In particular, under the assumption of bounded uncertainty in a collision avoidance task
with a lower bound on the constraint,
the non-convex optimization \cite{cheng2020safe} transforms into a quadratic program.
However, this approach requires exploiting particular details of
the multi-agent system.

RL, on the other hand, can excel
in environments that are non-convex.
While RL can induce unsafe policies, there are a variety of
approaches to improve safety characteristics.
Nevertheless, while these approaches can yield policies
that are more sensitive to safety considerations, safety violations
can still occur.
In \cite{achiam2017constrained}, the authors show how to efficiently
evaluate surrogate cost and objective functions with bounded errors
for a proposed policy based on data from a different policy
and incorporate this insight into a trust region optimization.
Lagrangian approaches \cite{ma2021feasible} similarly
seek to solve a constrained optimization where reward is maximized
subject to cost limits. In this approach, a cost coefficient
is adaptively updated so that costs stay below an upper bound.
This approach was used with barrier functions for instance in \cite{yang2023model}
where a learned barrier function constraint is used instead of a discounted
cost estimate. Similarly, \cite{luo2021learning} learns a barrier function
that satisfies some of the required properties by design and trains
an override policy to be safe with respect to the learned barrier function.
To induce exploration, they add noise to this policy and if after resampling
there are no safe actions, revert to the certified safe policy.
Action masking is also investigated based on cost thresholds
in \cite{thananjeyan2021recovery} where the override policy
is trained to minimize the expected discounted cost.
Similarly, \cite{srinivasan2020learning} also includes a Lagrangian
term to induce the nominal policy to include cost considerations.
In \cite{wagener2021safe}, the authors note that applying overrides that
are not significantly safer than the nominal policy can lead to worse
performance so they therefore require that the override provide an improvement over a threshold
in order to be used. Rather than explicitly computing a safe
override policy, \cite{zhang2022saferl} computes an online adjustment
to actions proposed by a nominal policy. In particular, given
a function estimating the discounted cost, they adjust the action
using gradient descent to find a safe override.
Unfortunately, none of these approaches can assure that the system
will stay safe throughout training or be safe 
at deployment after training is
completed.

This paper addresses how to enable safe control overrides
using barrier functions for systems that are non-convex
in the control input.
It makes the following contributions.
First, we derive barrier functions for two such systems,
namely fixed wing aircraft performing waypoint following
subject
to safe flight envelope constraints and self driving cars
doing lane merging with adaptive cruise control.
Second, we develop RL environments for both fixed wing aircraft
and self driving cars.
Third, we provide an experimental examination of the effect
of different safety override approximations on RL performance
as well as compare to safe RL baselines.
In particular, we find that all of the discussed overrides
yield performance commensurate with the highest performing baselines
without any safety violations. This holds even when the safety override
makes no attempt to match the nominal control value output.

This paper is organized as follows. In Section \ref{sec_background}, we 
provide background for barrier functions.
In Section
\ref{sec_examples_fw},
\ref{sec_examples_dbl_int},
and 
\ref{sec_examples_car},
we derive barrier functions for fixed-wing aircraft.
a double integrator,
and self-driving cars, repectively.
We introduce RL environments to compare a barrier function
override for fixed wing aircraft and self-driving cars
with simulated experimental results in Section \ref{sec_experiments}.
Section \ref{sec_conclusion} concludes.

\section{Background}

\label{sec_background}

Background for RL can be found in \cite{sutton2018reinforcement}.
Discrete time barrier functions were
analyzed in \cite{agrawal2017discrete}. Let
the system have dynamics
\begin{equation}
s_{k+1} = f(s_k, u_k)
\label{eq_dyn}
\end{equation}
where $s_k\in \R^{n_s}$ and $u_k\in U\subseteq \R^{n_u}$.
An output function $h:\R^{n_s}\to \mathbb{R}$
with superlevel set 
\begin{equation}
C_h = \{s_k\::\: h(s_k) \ge 0\}
\label{eq_safe_set}
\end{equation}
can be used to ensure the safety of a system.
Definition \ref{def_ecbf}
and Proposition \ref{def_ecbf}
are from \cite{squires2021model}
which rephrases Definition 4 and Proposition 4 of 
\cite{agrawal2017discrete} 
using terminology from \cite{ames2016control}. 
\begin{definition}
A map $h : \R^{n_s} \to \R$ is a Discrete-Time
Exponential Control Barrier Function (DT-ECBF) on a set
$D_h\subseteq \mathbb{R}^{n_u}$ where $C_h \subseteq D_h$
with $C_h$ defined in \eqref{eq_safe_set}
if for all $s_k\in D_h$ there is a $u_k \in\R^{n_u}$ and $\lambda\in (0,1]$ such that
\begin{equation}
c_h(s_k, u_k) \triangleq h(f(s_k, u_k)) - (1 - \lambda) h(s_k) \ge 0.
\label{eq_bf_constraint}
\end{equation}
\label{def_ecbf}
\end{definition}
The admissible
control space is 
\[K_h(s_k) = \{u_k\in U\::\: c_h(s_k,u_k) \ge 0\}.\]
\begin{proposition}
Given a set $C_h \subset \R^{n_s}$ defined in \eqref{eq_safe_set} for an
output function $h$, let $h$ be a DT-ECBF on $D_h$ and $u : \R^{n_s} \to U$
be such that $u(s_k) \in K_h(s_k)$ for all $s_k \in D_h$. Then if $s_0 \in C_h$
then $s_k\in C_h$ for all $k > 0$.
\label{prop_ecbf}
\end{proposition}
In \cite{ames2016control}, the authors introduce an optimization that solves for a safe control value
that minimizes the squared distance to
a nominal control value $\hat{u}_k$. If $h$
is non-convex, then the following can be non-convex
even when the dynamics \eqref{eq_dyn} are linear \cite{agrawal2017discrete}:
\begin{IEEEeqnarray}{rCCl}
u^*_k
&=&\argmin_{u_k\in\mathbb{R}^{n_u}} &\frac{1}{2}||u_k - \hat{u}_k||^2 \label{eq_prgm}\\
&&\text{s.t.} & c_h(s_k,u_k) \ge 0 \IEEEnonumber\\
&&& u_k\in U.\IEEEnonumber
\end{IEEEeqnarray}

In \cite{squires2022composition} the authors develop an approach
for generating a barrier function given a safety function
$\rho: \mathbb{R}^{n_s}\to \mathbb{R}$
which encodes a safety specification.
In
\cite{squires2021model} this approach is adapted to discrete
time systems. 
Given a safety function
$\rho$
and evasive maneuver
$\zeta:\mathbb{R}^{n_s}\to U$, a DT-ECBF can be constructed via
\begin{equation}
h(s_0) = \inf_{k\ge 0}\rho(\hat{s}_k)
\label{eq_bf}
\end{equation}
with $\hat{s}_0 = s_0$ and $\hat{s}_{k+1} = f(\hat{s}_k,\zeta(\hat{s}_k))$.

\section{A Barrier Function For Fixed Wing Aircraft}

\label{sec_examples_fw}

RL for waypoint following has been previously
discussed e.g., in \cite{tang2024trajectory}, although they did not
consider safety.
Composition of multiple constraints using barrier functions
has been developed for fixed wing aircraft
for collision avoidance \cite{squires2022composition}
as well as collision avoidance with a geofence constraint \cite{molnar2024collision}.
Here though we consider a model from 
\cite{boskovic2004adaptive} that is non-convex
in the control input.
The model considers
UAV
speed, pitch, and heading. Because we examine a waypoint following
problem in Section \ref{sec_experiments} we also include position dynamics.
The discrete time dynamics of the model in \cite{boskovic2004adaptive}
for discrete time
is then given by
\begin{IEEEeqnarray}{rCl}
\IEEEyesnumber
\label{eq_fw_dynamics}
\IEEEyessubnumber v_{k+1} &=& v_k + \delta g ([T_k - \mathcal{D}(v_k,n_k)]/W - \sin(\gamma_k)) \label{eq_fw_dynamics_v}\\
\IEEEyessubnumber \gamma_{k+1} &=& \gamma_k + \delta g(n_k\cos(\mu_k) - \cos(\gamma_k))/v_k \label{eq_fw_dynamics_gamma}\\
\IEEEyessubnumber \psi_{k+1} &=& \psi_k + \delta gn_k\sin(\mu_k) / (v_k\cos(\gamma_k)) \\
\IEEEyessubnumber x_{k+1} &=& x_{k} + \delta v_k\cos(\gamma_k) \cos(\psi_k) \\
\IEEEyessubnumber y_{k+1} &=& y_{k} + \delta v_k\cos(\gamma_k) \sin(\psi_k) \\
\IEEEyessubnumber z_{k+1} &=& z_{k} + \delta v_k\sin(\gamma_k)
\end{IEEEeqnarray}
where $\delta > 0$ is a timestep,
$g$ is gravity, $(v_k, \gamma_k,\psi_k)^T$ are speed,
pitch, and heading while $(x_{k},y_{k},z_{k})^T$
are positions.
The state is then $s_k = (v_k, \gamma_k, \psi_k, x_k, y_k, z_k)^T$
while the control input $u_k = (T_k,n_k,\mu_k)^T$
consists of the thrust, load factor, and bank angle
where load factor is the ratio of weight to lift force.
The control input $u_k$ is subject to actuator constraints
\begin{IEEEeqnarray*}{l}
U = \{u\in\mathbb{R}^{n_u}\:: 
T_k\in [0, T_{max}],
n_k\in [n_{min},n_{max}], \\
\qquad\qquad\qquad\qquad \mu_k\in [-\mu_{max}, \mu_{max}]\}.
\IEEEyesnumber
\label{eq_U}
\end{IEEEeqnarray*}
Finally, 
\begin{equation}
\mathcal{D}(v_k,n_k) = 0.5\mathcal{R} v_k^2AC_{D_0} + 2\mathcal{K}n_k^2W^2/(\mathcal{R} v_k^2 \mathcal{A})
\label{eq_drag}
\end{equation}
is the drag force where $\mathcal{R}$ is the air density,
$\mathcal{A}$ is the wing surface area, $C_{D0}$ is the parasitic
drag coefficient, $\mathcal{K}$ is the induced drag coefficient,
and $W$ is the weight of the vehicle.
This model is nonlinear in $n_k$
and non-convex in $\mu_k$.

The safety condition encompasses
stall conditions, excessive speed that can stress the aircraft platform,
and maintaining an altitude above a floor $z_{min}$.
Given $v_{min} > 0$, let
\begin{IEEEeqnarray}{l}
S =
\{s_k\in\mathbb{R}^{n_s}\::\:
\gamma_k\in[\gamma_{min},\gamma_{max}],
v_k\in[v_{min},v_{max}] \IEEEnonumber \\
\qquad\qquad\qquad\qquad z_{k} \ge z_{min}\}.
\label{eq_fw_safety_condition}
\end{IEEEeqnarray}

Let
\begin{IEEEeqnarray*}{l}
\alpha(s_k) = \min(\gamma_{max}\lambda v_k/\delta, gn_{max} - g), \label{eq_alpha}\IEEEyesnumber\\
\N(s_k) =
  \cos(\gamma_k)
  + \frac{1}{g}\min\left(\frac{\lambda v_k}{\delta}(\gamma_{max} - \gamma_k), \alpha(s_k)\right)  \label{eq_n_safe}\IEEEyesnumber\\
\T(s_k) = W\sin(\gamma_k) + \mathcal{D}(v_k,\widetilde{n}(s_k)) \label{eq_T_safe}\IEEEyesnumber\\
\widetilde{\mu}(s_k) = 0 \label{eq_mu_safe}\IEEEyesnumber\\
\U(s_k) = (\widetilde{T}(s_k), \widetilde{n}(s_k), \widetilde{\mu}(s_k))^T \label{eq_fw_u_safe}\IEEEyesnumber \\
\tau(s_k) = -\gamma_k/(\alpha(s_k)/v_k) \label{eq_tau_k}\IEEEyesnumber
\end{IEEEeqnarray*}
Denote $\alpha_k = \alpha(s_k)$,
$\T_k = \T(s_k)$, $\N_k = \N(s_k)$, $\widetilde{\mu}_k = \widetilde{\mu}(s_k)$,
$\U_k = \U(s_k)$, and $\tau_k = \tau(s_k)$.
For some intuition on $\alpha$ and $\tau$,
note that
if pitch is negative, i.e. $\gamma_k < 0$, then the load
factor $\N(s_k) = \cos(\gamma_k) + \alpha(s_k) / g$
in \eqref{eq_n_safe}.
Then
$\gamma_{k+1} - \gamma_k = \delta \alpha(s_k) / v_k$
in \eqref{eq_fw_dynamics_gamma}
so that 
$\alpha(s_k)/v_k$ can be interpreted as the change in pitch per unit time.
Noting that
$v_{k+1} = v_k$ in \eqref{eq_fw_dynamics_v}
given $\T$ in \eqref{eq_T_safe}
so $\alpha(s_k)$ is constant,
we can then interpret
$\tau$ as the time to reach zero pitch.

Let
\begin{IEEEeqnarray}{l}
\IEEEyesnumber
\label{eq_b}
\IEEEyessubnumber b_1(s_k) = v_{max} - v_k \label{eq_b1}\\
\IEEEyessubnumber b_2(s_k) = v_k - v_{min} \label{eq_b2}\\
\IEEEyessubnumber b_3(s_k) = \gamma_{max} - \gamma_k \label{eq_b3}\\
\IEEEyessubnumber b_4(s_k) = \gamma_k - \gamma_{min} \label{eq_b4}\\
\IEEEyessubnumber b_5(s_k) = \\
\qquad z_{k} + v_k\min(0,\gamma_k)(\max(\tau_k,0) + \delta) - z_{min}  \IEEEnonumber
\end{IEEEeqnarray}
We then define a candidate barrier function
\begin{equation}
h_{fw}(s_k) = \min_{i\in\{1,\ldots,5\}} b_i(s_k). \label{eq_bf_fw}
\end{equation}
The following can be computed with a straightfoward application of derivatives
to \eqref{eq_drag} and \eqref{eq_T_safe}:
\begin{IEEEeqnarray*}{l}
\T_{max} =
\max_{x\in\mathbb{R}^{n_s}\::\: v_k\in[v_{min},v_{max}],\gamma_k\in [\gamma_{min},\gamma_{max}]}\T(s_k)\IEEEyesnumber \label{eq_T_min}\\
\T_{min} =
\min_{x\in\mathbb{R}^{n_s}\::\: v_k\in[v_{min},v_{max}],\gamma_k\in [\gamma_{min},\gamma_{max}]}\T(s_k)\IEEEyesnumber \label{eq_T_max}
\end{IEEEeqnarray*}

For the candidate $h_{fw}$ to be useful for keeping the system
safe, we must have
$C_{h_{fw}}\subseteq S$ where $C_{h_{fw}}$ is defined in \eqref{eq_safe_set}
using $h_{fw}$. Hence we first show
conditions under which $b_i(s_k) \ge 0$ for $i=1,\ldots, 5$ implies $s_k\in S$. 
We then show when $\widetilde{u}_k\in U$ and
$b_i(f(s_k,\widetilde{u}_k) + (1 - \lambda)b_i(s_k) \ge 0$ for $i=1,\ldots,5$.
Finally, we give sufficient conditions under which a minimum
of functions
is a barrier function so that $h_{fw}$ in
\eqref{eq_bf_fw} is a barrier function.

\begin{lemma}
Suppose $v_{min} > 0$, $n_{max} > 1$, $\gamma_{max} > 0$,
and $b_2(s_k) \ge 0$.
Then $\alpha_k > 0$.
\label{lem_alpha}
\end{lemma}
\begin{proof}
This holds by direct substitution into \eqref{eq_alpha}.
\end{proof}

\begin{theorem}
Let the assumptions of Lemma \ref{lem_alpha} hold.
For $s_k\in \mathbb{R}^{n_s}$, if $b_i(s_k) \ge 0$ for $i=1,\ldots,5$ defined
in \eqref{eq_b},
$s_k\in S$ where $S$ is defined in \eqref{eq_fw_safety_condition}.
\label{th_fw_bf_useful}
\end{theorem}
\begin{proof}
From the definition of $b_i$ for $i=1,\ldots,4$,
$\gamma_k \in [\gamma_{min}, \gamma_{max}]$
and
$v_k\in [v_{min}, v_{max}]$ so $v_k \ge v_{min} > 0$.
We now show that $z_{k} \ge z_{min}$.
If $\gamma_{k} < 0$
then because $\alpha_k > 0$ by Lemma \ref{lem_alpha},
$\tau_k > 0$. Then
$0
\le b_5(s_k) = z_{k} + v_k\gamma_k(\tau_k + \delta) - z_{min}
\le z_{k} - z_{min}
$.
If $\gamma_k \ge 0$ then
$0 \le b_5(s_k) = z_k - z_{min}$.
Then $s_k \in S$.
\end{proof}

\begin{lemma}
Let the assumptions of Lemma \ref{lem_alpha} hold.
Let $\T_{min}$ and $\T_{max}$ be defined
in \eqref{eq_T_min} and \eqref{eq_T_max}, respectively,
and suppose
$0 \le \T_{min}$,
$\T_{max} \le T_{max}$,
and
$n_{min} \le \min(\cos(\gamma_{min}),\cos(\gamma_{max}))$.
If $s_k\in\mathbb{R}^k$ satisfies $b_i(s_k)\ge 0$ for $i\in\{1,\ldots,4\}$
then for $\U$ and $U$ defined in \eqref{eq_fw_u_safe}
and \eqref{eq_U}, respectively,
$\U(s_k)\in U$.
\label{lem_u_in_U}
\end{lemma}
\begin{proof}
From \eqref{eq_mu_safe},
$\widetilde{\mu}_k = 0 \in [-\mu_{max},\mu_{max}]$
and by assumption $\T_k\in[0,T_{max}]$.
By Lemma \ref{lem_alpha}, $\alpha_k > 0$.
By assumption $b_3(s_k) \ge 0$ and $b_4(s_k) \ge 0$
so $\gamma_k \in [-\gamma_{min}, \gamma_{max}]$.
Then from \eqref{eq_n_safe} we have
$\widetilde{n}_k \ge \cos(\gamma_k) \ge \min(\cos(\gamma_{min}),\cos(\gamma_{max})) \ge n_{min}$.
Similarly, from \eqref{eq_alpha} and \eqref{eq_n_safe},
$\widetilde{n}_k
\le \cos(\gamma_k) + \alpha_k / g
\le \cos(\gamma_k) + (gn_{max} - g)/g
= \cos(\gamma_k) + n_{max} - 1
\le n_{max}$.
\end{proof}

\begin{lemma}
Let the assumptions of Lemma \ref{lem_u_in_U} hold.
For the system \eqref{eq_fw_dynamics},
if $b_i(s_k) \ge 0$ for $i=1,\ldots,5$ in \eqref{eq_b}
where $s_k\in \mathbb{R}^{n_s}$
then
for $i=1,\ldots,5$
\begin{equation}
b_i(f(s_k,\U(s_k))) - (1 - \lambda) b_i(s_k) \ge 0.
\label{eq_b_constraint}
\end{equation}
where $\U(s_k)$ is defined in \eqref{eq_fw_u_safe}.
\label{lem_b_constraint_satisfied}
\end{lemma}
\begin{proof}
Denote $\U_k\triangleq \U(s_k)$
which is in $U$ by Lemma \ref{lem_u_in_U}.
Given $\widetilde{u}_k$,
$v_{k+1} = v_k$ so \eqref{eq_b_constraint}
is satisfied for $i=1,2$. Since
from Lemma \ref{lem_alpha} we have $\alpha_k > 0$
and the assumption $b_3(s_k) \ge 0$ implies $\gamma_k \le \gamma_{max}$,
we have $\widetilde{n}_k \ge \cos(\gamma_k)$ which implies
$\gamma_{k+1} \ge \gamma_k$ so \eqref{eq_b_constraint} is satisfied for $i=4$.
For $i=3$ and again given $\widetilde{u}_k$ in \eqref{eq_fw_u_safe},
$\gamma_{k+1}
\le \gamma_k + \lambda (\gamma_{max} - \gamma_k)$,
so that
\begin{IEEEeqnarray*}{l}
b_3(f(s_k, \U_k)) - (1 - \lambda)b_3(s_k) \\
\ge \gamma_{max} - \left(\gamma_k + \lambda (\gamma_{max} - \gamma_k)\right)
    - (1 - \lambda)(\gamma_{max} - \gamma_k) \\
= 0.
\end{IEEEeqnarray*}

We now show \eqref{eq_b_constraint} holds for $i=5$.
Suppose $0\le \gamma_k \le \gamma_{max}$.
Then because $v_k \ge v_{min} > 0$, $z_{k+1} \ge z_k$.
Further, since $\alpha_k > 0$ from Lemma \ref{lem_alpha},
we have $\N_k > \cos(\gamma_k)$ so
$\gamma_{k+1} \ge \gamma_k \ge 0$.
Then $b_5(f(s_k,\U_k)) = z_{k+1} - z_{min} \ge z_k - z_{min} \ge b_5(s_k)$
so \eqref{eq_b_constraint} holds for this case.

Consider then $\gamma_{min} \le \gamma_{k} < 0$.
In this case we have
$\alpha_k \le \lambda v_k (\gamma_{max} - \gamma_k) / \delta$
so
$\N_k = \cos(\gamma_k) + \alpha_k / g$
which implies that
$\gamma_{k+1} = \gamma_k + \delta \alpha_k/v_k$.
Also note $\alpha_k >0$ from Lemma \ref{lem_alpha}.

If $\gamma_k \le -\delta \alpha_k/v_k$
then
$\tau_k \ge \delta$,
$\tau_{k+1}
= -\gamma_{k+1} / (\alpha_k/v_k)
= -(\gamma_k + \delta \alpha_k/v_k) / (\alpha_k/v_k)
= \tau_{k} - \delta
$,
and
$b_5(s_k) = z_k + v_k\gamma_k(\tau_k+\delta) - z_{min}$.
Recall that for $\widetilde{u}_k$ we have $v_{k+1} = v_k$,
from Lemma \ref{lem_alpha} that $\alpha_k > 0$,
and note that since $\gamma_k < 0$ we have $\sin\gamma_k > \gamma_k$.
Then
\begin{IEEEeqnarray*}{l}
b_5(f(s_k, \U_k)) \\
= z_k + \delta v_k\sin\gamma_k + v_k(\gamma_k + \delta \alpha_k/v_k)(\tau_k - \delta + \delta) - z_{min} \\
\ge z_k + \delta v_k\gamma_k + v_k\gamma_k \tau_k - z_{min} \\
= b_5(s_k).
\end{IEEEeqnarray*}
For $-\delta \alpha_k/v_k < \gamma_k < 0$,
we have
$0 < \tau_k < \delta$.
Further,
$\gamma_{k+1} = \gamma_k + \delta \alpha_k / v_k \ge -\delta \alpha_k/v_k + \delta \alpha_k/v_k = 0$
so
$\tau_{k+1} \le 0$.
Then
\begin{IEEEeqnarray*}{rCl}
b_5(f(s_k, \U_k))
&=& z_k + \delta v_k\sin\gamma_k - z_{min} \\
&\ge& z_k + \delta v_k \gamma_k - z_{min} \\
&\ge& z_k + (\delta + \tau_k) v_k \gamma_k - z_{min} \\
&=& b_5(s_k).
\end{IEEEeqnarray*}
\end{proof}

Barrier function composition has been
previously discussed for continuous time
in \cite{glotfelter2017nonsmooth, wang2016multi}.
In \cite{squires2021model}, the authors showed
that the maximum of barrier functions is a barrier function.
For the case of a minimum of barrier functions, we need an additional condition.
This extra 
condition is similar to the shared evasive maneuver assumption
discussed in \cite{squires2022composition}.

\begin{lemma}
Let $q_i$ for $i=1,2$ be real valued functions on $\mathbb{R}^{n_s}$
with $q_3$ defined by $q_3(s_k) = \min_{i=1,2}q_i(s_k)$.
Choose $s_k\in \mathbb{R}^{n_s}$ and assume there exists some $u_k$ such that
$q_i(f(s_k, u_k)) - (1-\lambda)q_i(s_k) \ge 0$ for $i=1,2$.
Then $q_3(f(s_k,u_k)) - (1 - \lambda)q_3(s_k) \ge 0$.
\label{lem_min_b}
\end{lemma}

\begin{proof}
Without loss of generality assume $q_1(s_k) \le q_2(s_k)$.
If $q_1(f(s_k,u_k)) \le q_2(f(s_k,u_k))$
then $c_{q_3}(s_k,u_k) = c_{q_1}(s_k,u_k)\ge 0$.
If $q_1(f(s_k,u_k)) > q_2(f(s_k,u_k))$
then 
\begin{IEEEeqnarray*}{rCl}
c_{q_3}(s_k,u_k)
&=& q_2(f(s_k,u_k)) - (1 - \lambda)q_1(s_k) \\
&\ge& q_2(f(s_k,u_k)) - (1 - \lambda)q_2(s_k) \\
&=& c_{q_2}(s_k,u_k) \ge 0.
\end{IEEEeqnarray*}
\end{proof}

\begin{corollary}
Let $h_1$ and $h_2$ be DT-ECBFs on $D_1$ and $D_2$, respectively.
Suppose that for any $s_k\in D_1\cap D_2$ there is a $u_k \in U$ such that
$c_{h_1}(s_k,u_k) \ge 0$
and 
$c_{h_2}(s_k,u_k) \ge 0$.
Then $h_3$ defined by $h_3(s_k) = \min_{i=1,2}h_i(s_k)$ is a DT-ECBF
on $D_1\cap D_2$.
\label{cor_min_h}
\end{corollary}

\begin{proof}
From Lemma \ref{lem_min_b},
for any $s_k\in D_1\cap D_2$, there is a $u_k\in U$
such that $c_{h_3}(s_k, u_k) \ge 0$.
Then $h_3$ is a DT-ECBF.
\end{proof}

\begin{theorem}
Let the assumptions of Lemma \ref{lem_b_constraint_satisfied} hold.
Then $h_{fw}$ in \eqref{eq_bf_fw} is a DT-ECBF.
\label{th_bf}
\end{theorem}
\begin{proof}
Let $s_k\in \mathbb{R}^{n_s}$ where $h_{fw}(s_k) \ge 0$.
Then $b_i(s_k) \ge 0$ for $i=1,\ldots,5$.
Then from Lemma \ref{lem_u_in_U}, $\U_k \in U$
and from Lemma \ref{lem_b_constraint_satisfied},
$b_i$ satisfies \eqref{eq_b_constraint} for $i=1,\ldots,5$.
Then from Lemma \ref{lem_min_b} and induction,
$h_{fw}$ satisfies \eqref{eq_bf_constraint}.
\end{proof}
\begin{remark}
$b_1$ and $b_2$ in \eqref{eq_b1}-\eqref{eq_b2} are not necessarily
barrier functions. For instance, let $v_k = v_{min}$ and
$\pi / 2 \ge \gamma_{k} > \gamma_{max}$.
Then $b_1(s_k) \ge 0$.
Suppose $k = C_{D0} = 0$ so that $\mathcal{D}(v_k,n_k) = 0$.
If $T_{max} = W\sin(\gamma_{max})$, 
then for any $u_k\in U$,
$v_{k+1} 
\le v_k + \delta g(T_{max} / W - \sin(\gamma_k))
= v_k + (\sin(\gamma_{max}) - \sin(\gamma_k))
< v_{min}
$.
A similar calculation shows $b_2$ is not a barrier function
when $v_k = v_{max}$ and $\gamma_k < -\gamma_{max}$.
Hence we use Lemma \ref{lem_min_b} rather than
Corollary \ref{cor_min_h} in the proof of Theorem \ref{th_bf}.
\end{remark}

\section{Barrier Function For A Double Integrator}

\label{sec_examples_dbl_int}

Barrier functions for continuous time
double integrator systems have been previously been analyzed in
\cite{borrmann2015control} and \cite{squires2021barrierPhd}.
A higher order integrator was analyzed
in \cite{wang2017safe}. Here we derive a barrier function
for discrete time where the double integrator dynamics are given by
\begin{IEEEeqnarray}{rCl}
\IEEEyesnumber
\label{eq_dbl_int_dynamics}
\IEEEyessubnumber p_{k+1} &=& p_{k} + \delta v_k \label{eq_dbl_int_dynamics_p}\\
\IEEEyessubnumber v_{k+1} &=& v_{k} + \delta a_k \label{eq_dbl_int_dynamics_v}
\end{IEEEeqnarray}
where the state $s_k=(p_{k},v_k)^T$ consists of position and velocity
while the input $a_k\in[a_{min},a_{max}]$ is the acceleration.

Although this double integrator is linear in the control input,
we use the results in this section for the non-convex self-driving
car system in the next section.
The floor function is denoted by $\lfloor \cdot \rfloor$.
For $\A = (a^-,a^+)^T\in\mathbb{R}^2$ with $a^-\in[a_{min}, 0)$ and $a^+\in(0,a_{max}]$,
let
\begin{IEEEeqnarray}{l}
\U_{dbl}(s_k,\A)
= \begin{cases}
    \max(a^-, -v_k / \delta)) & v_k \ge 0 \\
    \min(a^+, -v_k / \delta)) & v_k < 0
    \end{cases}
    \label{eq_dbl_int_u_k} \\
A(s_k, \A) = \begin{cases} a^- & v_k \ge 0 \\ a^+ & v_k < 0 \end{cases} \\
N(s_k,\A)
    = \lfloor |v_k| / |\delta A(s_k, \A) |\rfloor \label{eq_dbl_int_N_k}
    \label{eq_N}\\
\eta(s_k,\A) =
p_k + \delta N(s_k,\A) v_k \label{eq_eta}\\
\qquad\qquad +\frac{N(s_k,\A)(N(s_k,\A) - 1)}{2}\delta^2 A(s_k, \A) \IEEEnonumber\\
\qquad \qquad + \delta  (v_k + \delta N(s_k,\A)A(s_k, \A)) \IEEEnonumber
\end{IEEEeqnarray}

\begin{lemma}
For a system with dynamics \eqref{eq_dbl_int_dynamics},
suppose for all integers $l\ge 0$ that
$a_{k+l} = \U_{dbl}(s_{k+l},\A)$
for $\A = (a^-,a^+)^T\in\mathbb{R}^2$ with $a^-\in[a_{min}, 0)$ and $a^+\in(0,a_{max}]$.
Then
\begin{enumerate}[a)]
\item $p_{k+l} = \eta(s_k,\A)$ for all $l > N(s_k,\A)$.
\item If $ v_k \ge 0$ then $p_{k+l} \le \eta(s_k,\A)$ for all $l \ge 0$.
    If $ v_k \le 0$ then $p_{k+l} \ge \eta(s_k,\A)$ for all $l \ge 0$.
\item $\eta$ is continuous in $s_k$ and $\eta(s_k,\A) = \eta(f(s_k,\U_{dbl}(s_k,\A)),\A)$.
\item $\eta$ increases as $v_k$ increases.
\end{enumerate}
\label{lem_dbl_int_eta}
\end{lemma}
\begin{proof}
Consider the case where $v_k \ge 0$ as the other case follows similarly.
We first show that for nonnegative integers $l$
\begin{equation}
v_{k+l} =
\begin{cases}
    v_k + l\delta a^- & l \le N(s_k,\A) \\
    0 & l > N(s_k, \A)
\end{cases}.
\label{eq_dbl_int_solution_v}
\end{equation}
This holds trivially for $l = 0$.
Suppose it holds for some $l$ satisfying $0 \le l \le N(s_k,\A) - 1$.
Then from \eqref{eq_N},
$v_k
\ge N(s_k,\A)\delta |a^-|
\ge (l+1)\delta |a^-|
= -(l+1)\delta a^-$.
Then from the induction hypothesis,
$v_{k+l} = v_k + l\delta a^- \ge -(l+1)\delta a^- + l\delta a^- = -\delta a^-$
so $-v_{k+l}/\delta \le a^-$
and $\U_{dbl}(s_{k+l}, \A) = a^-$.
Then $v_{k+l+1} = v_{k+l} + \delta a^- = v_k + (l+1)\delta a^-$
so \eqref{eq_dbl_int_solution_v} holds for $l \le N(s_k,\A)$.
When $l = N(s_k,\A)$, we have
\begin{equation}
v_{k+l} = v_k + N(s_k,\A)\delta a^- = v_k + \lfloor v_k / (\delta |a^-|) \delta a^-.
\label{eq_v_k_plus_l}
\end{equation}
From \eqref{eq_v_k_plus_l} we then have
\begin{equation}
v_{k+l} \ge v_k + (v_k / \delta |a^-|)\delta a^- = 0.
\label{eq_v_k_plus_l_gt}
\end{equation}
Similarly, from \eqref{eq_v_k_plus_l}
and recalling that $a^- < 0$ so 
$\lfloor v_k/(\delta |a^-|)\rfloor \delta a^- \le (v_k/(\delta |a^-|) - 1)\delta a^-$,
\begin{equation}
v_{k+l} \le v_k + (v_k / (\delta |a^-|) - 1)\delta a^- = -\delta a^-.
\label{eq_v_k_plus_l_lt}
\end{equation}
Then from \eqref{eq_v_k_plus_l_gt}, \eqref{eq_v_k_plus_l_lt},
and \eqref{eq_dbl_int_u_k},
we have $\U_{dbl}(s_{k+l}) = -v_{k+l} / \delta$
so $v_{k+l+1} = v_{k+l} - \delta v_{k+l}/\delta = 0$
and it follows trivially by induction that $v_{k+l} = 0$
for all $l > N(s_k, \A)$.

We now show by induction that 
for nonnegative integers $l$ with $l \le N(s_k, \A)$,
\begin{equation}
p_{k+l} = p_k + \delta l v_k + \frac{l(l - 1)}{2}\delta^2  a^-.
\label{eq_dbl_int_solution_p}
\end{equation}
This holds trivially for $l = 0$.
Suppose it holds for some $l$ where $0 \le l \le N(s_k,\A) - 1$.
Then since $l(l-1)/2 + l = (l+1)l / 2$,
\begin{IEEEeqnarray*}{rCl}
p_{k+l+1}
&=& p_{k+l} + \delta  v_{k+l} \\
&=& p_k + \delta l  v_k + \frac{l(l-1)}{2}\delta^2  a^- + \delta  (v_k + \delta la^-) \\
&=& p_k + \delta (l+1)  v_k + \frac{(l+1)l}{2}\delta^2  a^-.
\end{IEEEeqnarray*}
Then \eqref{eq_dbl_int_solution_p} holds for $0 \le l \le N(s_k,\A)$.
For $l = N(s_k,\A)$, note that
$p_{k+l+1} = p_{k+l} + \delta  v_{k+l} = p_{k+l} + \delta  (v_k + \delta N(s_k,\A)a^-) = \eta(s_k,\A)$.
Since $v_{k+l+1} = 0$ for all $l > N(s_k,\A)$,
$p_{k+l} = p_{k+N(s_k,\A)+1}$.
Since $v_{k+l} \ge 0$ for all $l\ge 0$,
$p_{k+l}$ is a non-decreasing function
so $p_{k+l} \le \eta(s_k,N_k,\A)$ for all $l \ge 0$.
Hence (a) and (b) hold.

We now show that $\eta$ is continuous.
Choose some positive integer $M$.
It has already been shown in part (a) that if $M > v_k / \delta a^- + 1$
that $\eta(s_k,\A) = p_{k+M}$.
But the mapping $p_k \to p_{k+M}$ using the control
law $\U_{dbl}$ in \eqref{eq_dbl_int_u_k}
is continuous because \eqref{eq_dbl_int_u_k}
and \eqref{eq_dbl_int_dynamics} are continuous.
Then $\eta$ is continuous on the set $\{s_k \in \mathbb{R}^{n_u}\::\: v_k / \delta a^- + 1 < M\}$.
Since $M$ can be chosen to be arbitrarily large,
$\eta$ is continuous on $\mathbb{R}^{n_s}$.

We now prove that $\eta(s_k,\A) = \eta(f(s_k,\U_{dbl}(s_k,\A)),\A)$.
Note again we are considering the case of
$v_k \ge 0$ since the other case follows similarly.
Denote $s_{k+1} \triangleq f(s_k,\U_{dbl}(s_k,\A))$.

Suppose $N(s_k,\A) = 0$.
Then $\eta(s_k) = p_k + \delta v_k$.
Also, since $N(s_k,\A)= 0$ we have $v_k < \delta |a^-|$, so given 
$\U_{dbl}$ in \eqref{eq_dbl_int_u_k},
$v_{k+1} = 0$.
Then $N(s_{k+1},\A) = 0$ so
$\eta(s_{k+1}) = p_{k} +\delta v_{k}$.

Suppose $N(s_k,\A) \ge 1$.
Then
$N(s_{k+1},\A)
= \lfloor |(v_k + \delta a^-)| / |\delta a^-|\rfloor
= \lfloor (v_k + \delta a^-) / |\delta a^-|\rfloor
= N(s_k,\A) - 1$.
Then
\begin{IEEEeqnarray*}{rCl}
\eta(s_{k+1},\A)
&=&
p_k + \delta  v_k + \delta (N(s_k,\A) - 1) (v_k + \delta a^-) \\
&& + \frac{(N(s_k,\A) - 1)(N(s_k,\A) - 2)}{2}\delta^2  a^- \\
&& + \delta  (v_k + \delta a^- + \delta (N(s_k,\A) - 1)a^-).
\end{IEEEeqnarray*}
Noting $(l-1)(l-2)/2 = l(l-1)/2 + 1 - l$, we have
\begin{IEEEeqnarray*}{l}
\eta(s_{k+1},\A) \\
= p_k + \delta v_k + \delta N(s_k, \A)v_k + \delta^2 N(s_k,\A)a^- - \delta v_k - \delta^2a^- \\
+ \frac{N(s_k,\A)(N(s_k,\A) - 1)}{2}\delta^2a^- + \delta^2 a^- - N(s_k,\A)\delta^2a^- \\
+ \delta (v_k + \delta N(s_k,\A)a^-) + \delta^2a^- - \delta^2a^- \\
= \eta(s_k,\A).
\end{IEEEeqnarray*}
Finally, we show that as $v_k$ increases, $\eta$ increases.
Consider $s_{a,k} = (p_{a,k},v_{a,k})^T$ and $s_{b,k} = (p_{b,k},v_{b,k})^T$
for $p_{a,k} = p_{b,k}$ and $v_{a,k} = v_{b,k} + \epsilon$ for some $\epsilon > 0$.
From \eqref{eq_dbl_int_solution_v},
$v_{a,k+l} \ge v_{v,k + l}$ for all integers $l\ge 0$.
Then since $p_{a,k} = p_{b,k}$ and
$v_{a,k+l} \ge v_{b,k+1}$ for all $l\ge 0$, $p_{a,k+l} \ge p_{b,k+l}$
for all $l\ge 0$.
Then $\eta(s_{a,k}) \ge \eta(s_{b,k})$
follows from Lemma \ref{lem_dbl_int_eta}a.
\end{proof}

By Lemma \ref{lem_dbl_int_eta}a, the output of $\eta$
is where the position ends after using $\U_{dbl}$ for all time.
Hence for $p_{min},p_{max}\in \mathbb{R}$, let
\begin{IEEEeqnarray}{l}
h_{L,dbl}(s_k) =  \min(p_k, \eta(s_k,\A)) - p_{min} \quad \label{eq_dbl_int_bf_low}\\
h_{H,dbl}(s_k) =  p_{max} - \max(p_k, \eta(s_k,\A)) \quad \label{eq_dbl_int_bf_high}.
\end{IEEEeqnarray}

\begin{theorem}
For the system \eqref{eq_dbl_int_dynamics},
$h_{L,dbl}$ and $h_{H,dbl}$ defined in \eqref{eq_dbl_int_bf_low} and \eqref{eq_dbl_int_bf_high}
are DT-ECBFs.
\label{th_dbl_int_bf}
\end{theorem}
\begin{proof}
Choose $s_k$ such that $h_{L,dbl}(s_k) \ge 0$.
Denote $s_{k+1}$ as the state after applying $\U_{dbl}(s_k,\A)$
given dynamics \eqref{eq_dbl_int_dynamics}.

If $v_k \ge 0$ then $p_{k+1} \ge p_k$.
Further, by Lemma \ref{lem_dbl_int_eta}b,
$p_k \le \eta(s_k,\A)$. Given
$a_k = \U_{dbl}(s_k,\A)$,
we have $v_{k+1} \ge 0$ so, again by Lemma \ref{lem_dbl_int_eta}b,
$p_{k+1} \le \eta(s_{k+1})$.
Then
$h_L(s_{k+1}) = p_{k+1} - p_{min} \ge p_k - p_{min} = h_L(s_k)$.

If $v_k < 0$ then similarly by Lemma \ref{lem_dbl_int_eta}b
and given $\U_{dbl}(s_k)$,
$p_k \ge \eta(s_k,\A)$,
$v_{k+1} \le 0$,
$p_{k+1} \ge \eta(s_{k+1},\A)$.
Then by Lemma \ref{lem_dbl_int_eta}c,
$h_{L,dbl}(s_{k+1}) = \eta(s_{k+1},\A) - p_{min} = \eta(s_{k},\A) - p_{min} = h_{L,dbl}(s_k)$.
The proof for \eqref{eq_dbl_int_bf_high} is similar.
\end{proof}
\begin{remark}
From Lemma \ref{lem_dbl_int_eta}c, $\eta$ is continuous.
Hence $h_{L,dbl}$ and $h_{H,dbl}$ are continuous.
\end{remark}

\section{Barrier Function for Lane Merging and Adaptive Cruise Control}

\label{sec_examples_car}

Consider a driving scenario where a vehicle is on a road with two lanes going
in the same direction. Note \cite{ames2016control} also analyzed this problem
for lane keeping and adaptive cruise control. The approach here differs in that
the barrier functions allow for changing lanes, account for multiple lead
cars in multiple lanes, the dynamics are not linear in the control input,
and the model is in discrete time. Deriving a barrier
function in discrete time can be challenging
because we cannot take derivatives
with respect to time to find the minimum of a function.
However, discrete time 
barrier functions do not need to be continuously
differentiable
so this provides some extra flexibility.
In particular, in this section we use the fact that
the maximum or minimum
of barrier functions is a barrier function.
See Corollary \ref{cor_min_h}.

Consider a case where there are
two lead vehicles as well
as an autonomous vehicle with states $s_{j,k}$ for $j=1,2,3$, respectively.
Index $j=3$ refers to the autonomous car while indexes $j=1,2$ refers
to lead cars in lanes $1$ and $2$, respectively.
The kinematic bicycle model vehicle with state $s_{j,k} = (x_{j,k}, y_{j,k}, v_{j,k}, \psi_{j,k})^T$
is given by (see equation (2) of \cite{polack2017kinematic})
\begin{IEEEeqnarray}{rCl}
\IEEEyesnumber
\label{eq_car_dyn}
\IEEEyessubnumber x_{j,k+1} &=& x_{j,k} + \delta v_{j,k}\cos(\psi_{j,k}+\beta(u_{2,j,k})) \\
\IEEEyessubnumber y_{j,k+1} &=& y_{j,k} + \delta v_{j,k}\sin(\psi_{j,k}+\beta(u_{2,j,k})) \\
\IEEEyessubnumber v_{j,k+1} &=& v_{j,k} + \delta u_{1,j,k} \\
\IEEEyessubnumber \psi_{j,k+1} &=& \psi_{j,k} + \delta \frac{v_{j,k}}{l_{j,r}}\sin(\beta(u_{2,j,k})) \label{eq_car_dyn_psi}
\end{IEEEeqnarray}
where $x_{j,k}$ ($y_{j,k}$) is the $x$ ($y$)-position,
$v_{j,k}$ is the speed,
$\psi_{j,k}$ is the angle between the centerline of the road and the vehicle velocity,
$\beta(u_{2,j,k}) = \tan^{-1}(\tan(u_{2,j,k})l_{j,r}/(l_{j,f}+l_{j,r}))$,
$u_{1,j,k}\in[a_{min},a_{max}]$ is the acceleration, $u_{2,j,k}\in[-u_{2,max},u_{2,max}]$
is the front wheel steering angle,
and $l_{j,f}$ ($l_{j,r}$) is the distance between the front (rear) axle and center of gravity.
The state of the system is $s_j = ([s_{1,k}]^T, [s_{2,k}]^T, [s_{3,k}]^T)^T$.

For brevity let $b_{min} = \sin(\beta(-u_{2,max}))$ and $b_{max} = \sin(\beta(u_{2,max}))$.
Then
$\beta^{-1}(\sin^{-1}(\cdot))$
where $\beta^{-1}(\cdot) = \tan^{-1}(\tan(\cdot)(l_{j,f} + l_{j,r})/l_{j,r})$
is well defined
for any value within $[b_{min}, b_{max}]$ given a small enough $u_{2,max}$.
Let $\U_{steer}(s_{3,k}) = 0$ for $v_{3,k} = 0$ and 
\begin{IEEEeqnarray*}{l}
\U_{steer}(s_{3,k}) = \label{eq_car_u_k_steer} \IEEEyesnumber\\
\quad
\beta^{-1}(\sin^{-1}(
\min\{
    \max[
        -\psi_{3,k} l_{3,r}/(\delta v_{3,k}),b_{min}
    ], b_{max}
\}
)) 
\end{IEEEeqnarray*}
otherwise. Let
\begin{IEEEeqnarray*}{l}
\U_3(s_{3,k}, \A_3) = (\U_{dbl}(s_{3,k},\A_3),\U_{steer}(s_{3,k}))^T
\label{eq_car_u_k} \IEEEyesnumber
\end{IEEEeqnarray*}
for $\A_3=[a_3^-, a_3^+]^T$ for $a_3^-\in [u_{1,min},0)$, $a_3^+\in(0,u_{1,max}]$ and
for $\U_{k,dbl}$ defined in \eqref{eq_dbl_int_u_k}. 

\begin{figure}
\centering
\def\svgwidth{0.90\columnwidth}
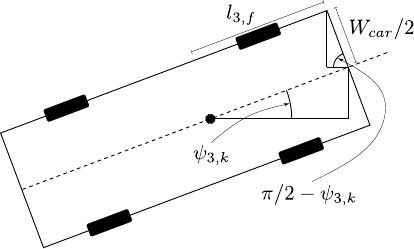
\caption{
Computing the lateral offset from the state
so that the edge of the car does not exit a lane.
}
\label{fig_car_offset}
\end{figure}

Given the car position of
$y_{3,k}$, the largest offset of the vehicle given a width $W_{car}$ 
is given by (see Figure \ref{fig_car_offset})
\begin{equation}
O(s_{3,k}) = l_{3,f}|\sin(\psi_{3,k})| + (W_{car}/2)|\cos(\psi_{3,k})|.
\label{eq_car_offset}
\end{equation}
To stay in lane 1,
the car must maintain $y_{3,k} \in [O(s_{3,k}), W_{lane} - O(s_{3,k})]$ where $W_{lane}$ is the width
of the lane.
Similarly, to stay in lane 2,
the car must maintain $y_{3,k}\in [W_{lane} + O(s_{3,k}),2W_{lane} - O(s_{3,k})]$.
Let
\begin{IEEEeqnarray*}{rCl}
\rho_{L^1}(s_k) &=& y_{3,k} - O(s_{3,k}) \\
\rho_{H^1}(s_k) &=& W_{lane} - O(s_{3,k}) - y_{3,k} \\
\rho_{L^2}(s_k) &=& y_{3,k} - O(s_{3,k}) - W_{lane} \\
\rho_{H^2}(s_k) &=& 2W_{lane} - O(s_{3,k}) - y_{3,k}.
\end{IEEEeqnarray*}
The subscript $L^i$ ($H^i$) means the lower (higher) boundary for lane $i$ for $i=1,2$.
Construct $h_{L^1},h_{H^1},h_{L^2},$ and $h_{H^2}$
via \eqref{eq_bf} using evasive maneuver $\zeta = \U_3$
in \eqref{eq_car_u_k}
and safety functions $\rho_{L^1},\rho_{H^1},\rho_{L^2}$, and $\rho_{H^2}$, respectively.
Then by Theorem 1 of \cite{squires2021model}, 
$h_{L^1},h_{H^1},h_{L^2},$ and $h_{H^2}$ are barrier functions.
While \eqref{eq_bf} requires an evaluation
over an infinite horizon, given $\zeta = \U_3$,
$v_{3,k+l}$ will be zero within $N((y_{3,k},v_{3,k})^T, \A)+1$ timesteps
so \eqref{eq_bf} can be evaluated over $N((y_{3,k},v_{3,k})^T, \A)+1$
timesteps.
If $v_{3,k}$ is small, then $N((y_{3,k},v_{3,k})^T, \A)$ will
be small. If $v_{3,k}$ is large then
$\psi_{3,k+l}$ will also be zero reasonably quickly
(see \eqref{eq_car_dyn_psi} and \eqref{eq_car_u_k_steer}).
See Figure \ref{fig_num_steps}
for a computation of the number of timesteps
required for a variety of settings of $v_{3,k}$ and $\psi_{3,k}$
The maximum number of steps is 33 over this range.

\begin{figure}
    \includegraphics{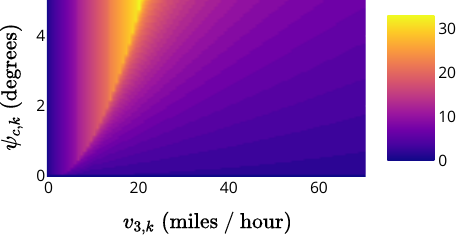}
    \caption{
        The number of steps to evaluate a barrier function
        in \eqref{eq_bf}
        for keeping a safe distance from lane boundaries.
    }
    \label{fig_num_steps}
\end{figure}

We now consider how to ensure that the car with state $s_{3,k}$ maintains
a safe distance from the lead cars $s_{1,k}$ and $s_{2,k}$, the cars in
lanes $1$ and $2$, respectively.
The \emph{lead car assumption} says that the lead cars stay in their
lane, maintain a non-negative speed, car with index $1$ and $2$ are in
lanes $1$ and $2$ respectively,
and the lead cars do not turn.

\begin{assumption}
(Lead Car Assumption) For all integers $k$,
$0 \le y_{1,k} \pm W_{car} \le W_{lane} \le y_{2,k} \pm W_{car} \le 2 W_{lane}$,
$\psi_{1,k}=\psi_{2,k} = 0$, and $v_{1,k} = v_{2,k} \ge 0$.
\label{assumption_lead_car}
\end{assumption}

Under the lead car assumption, $u_{2,j,k} = 0$
so the dynamics of lead car $j=1,2$
is given by
\begin{IEEEeqnarray*}{rCl}
x_{j,k+1} &=& x_{j,k} + \delta v_{j,k} \\
v_{j,k+1} &=& v_{j,k} + \delta u_{1,j,k}.
\end{IEEEeqnarray*}

Let $D_{lead} \ge 0$ be the 
minimum allowed intervehicle distance
and $\tau_{head}$ the minimum allowed headway
so that the distance from the lead
and following car must satisfy $(x_{j,k} - x_{3,k} - D_{lead}) - v_{3,k}\tau_{head}$
for all $k$ and $j=1,2$.
The time headway distance is a common metric
for safe distances in both barrier functions
and MPC \cite{ames2016control,katriniok2023discrete}.
For $j=1,2$ with $D_{lead} \ge 0$ and $\tau_{head} \ge 0$, let
\begin{IEEEeqnarray*}{rCl}
\widetilde{v}_{3,k} &=& \max(0, v_{3,k}) \\
d_{j,lead}(s_k)
    &=& \eta((x_{j,k},\min(v_{j,k},\widetilde{v}_{3,k}))^T, (u_{1,min},u_{1,max})) \\
   && \qquad - \eta((x_{3,k},\widetilde{v}_{3,k})^T,\A_3) \\
h_{j,lead}(s_k)
    &=& d_{j,lead}(s_k) - D_{lead} - \widetilde{v}_{3,k} \tau_{head}
\IEEEyesnumber \label{eq_car_lead_bf}
\end{IEEEeqnarray*}
\begin{theorem}
For a system with state $s_k = ([s_{1,k}]^T, [s_{2,k}]^T, [s_{3,k}]^T)^T$ where $s_{j,k}$ 
has dynamics in \eqref{eq_car_dyn} for $j=1,2,3$,
suppose that the lead car assumption holds. Then
$h_{j,lead}$ defined in \eqref{eq_car_lead_bf} for $j=1,2$
is a DT-ECBF.
\label{th_car_bf}
\end{theorem}
\begin{proof}
For notational simplicity, we set $j=1$ in this proof.
The case of $j=2$ follows similarly.
Denote
$\eta_{1,k}(v) \triangleq \eta((x_{1,k},v)^T,(u_{1,min},u_{1,max}))$
and
$\eta_{3,k}(v) \triangleq \eta((x_{3,k},v)^T,\A_3)$.
By the lead car assumption, $\psi_{1,k} = 0$ for all $k$.
Vehicle $3$ does not have control over the lead vehicles
so we assume the worst case where $u_{1,j,k} = \U_{dbl}(s_{1,k},\A_1)$ with $\A_1 = (u_{1,min},u_{1,max})^T$.
Denote
$v_{1,k+1} = v_{1,k} + \delta \U_{dbl}((x_{1,k},v_{1,k})^T,\A_1)$
and 
$v_{3,k+1} = v_{3,k} + \delta \U_{dbl}((x_{3,k},v_{3,k})^T,\A_3)$.
We show that when vehicle $3$ has control input given by \eqref{eq_car_u_k},
$h_{1,lead}$ satisfies \eqref{eq_bf_constraint}.

\emph{Case 1:}
Suppose $v_{3,k} \le 0$.
Then with control input given by \eqref{eq_car_u_k},
$v_{3,k+1} \le 0$ so $\V_{3,k} = \V_{3,k+1} = 0$.
By the lead car assumption, $v_{1,k} \ge 0$ so 
$x_{1,k+1} \ge x_{1,k}$ so $\eta_{1,k+1}(0) = x_{1,k+1} \ge x_{1,k} = \eta_{1,k}(0)$.
For similar reasons since $v_{3,k} \le 0$, $\eta_{3,k+1}(0) \le \eta_{3,k}(0)$.
Then
$h_{1,lead}(s_{k+1})
= \eta_{1,k+1}(0) - \eta_{3,k+1}(0) - D_{lead}
\ge \eta_{1,k}(0) - \eta_{3,k}(0) - D_{lead}
= h_{1,lead}(s_k)$.

\emph{Case 2:}
Suppose $v_{3,k} \ge 0$ so that given the control in \eqref{eq_car_u_k},
$v_{3,k} \ge v_{3,k+1} \ge 0$.
From Lemma \ref{lem_dbl_int_eta}c,
$\eta_{1,k+1}(v_{1,k+1}) = \eta_{1,k}(v_{1,k})$
and $\eta_{3,k+1}(v_{3,k+1}) = \eta_{3,k}(v_{3,k})$.

\emph{Case 2a:}
Suppose $v_{1,k} \le v_{3,k}$.
Then $v_{1,k+1} \le v_{3,k+1}$ because $\U_{dbl}((x_{3,k},v_{3,k})^T,\A_3) \ge \U_{dbl}((x_{1,k},v_{1,k})^T,\A_1)$. Then
\begin{IEEEeqnarray*}{l}
h_{1,lead}(s_{k+1}) \\
= \eta_{1,k+1}(v_{1,k+1}) - \eta_{3,k+1}(v_{3,k+1}) - D_{lead} - \tau_{head} v_{3,k+1} \\
\ge \eta_{1,k}(v_{1,k}) - \eta_{3,k}(v_{3,k}) - D_{lead} - \tau_{head} v_{3,k} \\
= h_{1,lead}(s_k).
\end{IEEEeqnarray*}

\emph{Case 2b:}
If $v_{1,k} \ge v_{3,k} \ge 0$
then $\eta_{1,k}(v_{1,k}) \ge \eta_{1,k}(v_{3,k})$ by Lemma \ref{lem_dbl_int_eta}d.
If $v_{1,k+1} < v_{3,k+1}$, then
using Lemma \ref{lem_dbl_int_eta}c,
\begin{IEEEeqnarray*}{l}
h_{1,lead}(s_{k+1}) \\
= \eta_{1,k+1}(v_{1,k+1}) - \eta_{3,k+1}(v_{3,k+1}) - D_{lead} -v_{3,k+1}\tau_{head} \\
\ge \eta_{1,k}(v_{1,k}) - \eta_{3,k}(v_{3,k}) - D_{lead} - v_{3,k}\tau_{head} \\
\ge \eta_{1,k}(v_{3,k}) - \eta_{3,k}(v_{3,k}) - D_{lead} - v_{3,k}\tau_{head} \\
= h_{1,lead}(s_k).
\end{IEEEeqnarray*}
Consider now when 
$v_{1,k+1} \ge v_{3,k+1}$.
For notational simplicity, denote
$\U_{dbl,1,k}(v) \triangleq \U_{dbl}((x_{1,k},v)^T,\widetilde{a}_1)$.
Then since $v_{1,k} > v_{3,k} \ge 0$,
\begin{equation}
\U_{dbl,1,k}(v_{1,k})
\le \U_{dbl,1,k}(v_{3,k}).
\label{eq_case2b_1}
\end{equation}
Then using \eqref{eq_case2b_1}
and that $\U_{dbl,1,k}(v_{3,k}) \le \U_{dbl}((x_{3,k},v_{3,k})^T,\A_3)$,
\begin{equation}
v_{3,k} + \delta \U_{dbl,1,k}(v_{1,k}) 
\le
v_{3,k} + \delta \U_{dbl,1,k}(v_{3,k})
\le
v_{3,k+1}.
\label{eq_case2b_2}
\end{equation}
From Lemma \ref{lem_dbl_int_eta}c,
$
\eta_{1,k}(v_{3,k})
=
\eta_{1,k+1}(v_{3,k} + \delta \U_{dbl,1,k}(v_{1,k}))
$.
From this, \eqref{eq_case2b_2}, and Lemma \ref{lem_dbl_int_eta}d we have
\begin{equation}
\eta_{1,k}(v_{3,k})
=
\eta_{1,k+1}(v_{3,k} + \delta \U_{dbl,1,k}(v_{1,k}))
\le 
\eta_{1,k+1}(v_{3,k+1}).
\label{eq_case2b_3}
\end{equation}
Then using \eqref{eq_case2b_3}, noting that $v_{3,k+1} \le v_{3,k}$,
and that Lemma \ref{lem_dbl_int_eta}c implies $\eta_{3,k}(v_{3,k}) = \eta_{3,k+1}(v_{3,k+1})$,
we have
\begin{IEEEeqnarray*}{l}
h_{1,lead}(s_k) \\
= \eta_{1,k}(v_{3,k}) - \eta_{3,k}(v_{3,k}) - D_{lead} - \tau_{head}v_{3,k} \\
\le \eta_{1,k+1}(v_{3,k+1}) - \eta_{3,k+1}(v_{3,k+1}) - D_{lead} - \tau_{head}v_{3,k+1} \\
= h_{1,lead}(s_{k+1}).
\end{IEEEeqnarray*}
\end{proof}

\noindent
The vehicle must also obey speed limits
so
let $v_{lim} \ge 0$ and
\begin{equation}
h_{spd}(s_k) = v_{lim} - \max(0, v_{3,k}) \label{eq_car_h_spd},
\end{equation}
which is a barrier function by a straightforward
calculation using the control input \eqref{eq_car_u_k}.

Given Corollary \ref{cor_min_h} and Theorem 3 of \cite{squires2021model}, we
now use barrier function composition to simultaneously satisfy the safety
constraints. For brevity, denote $\max(\cdot, \cdot)$ by $\cdot \wedge \cdot$
and $\min(\cdot,\cdot)$ as $\cdot \vee \cdot$. Additionally,
we drop the function argument so $h = h_{spd} \wedge h_{turn}$
means that given $s_k\in\mathbb{R}^{n_s}$, $h(s_k) = h_{spd}(s_k) \wedge h_{turn}(s_k)$.
Let
\begin{IEEEeqnarray*}{l}
h_{car}
= \\
\Big[h_{spd} \wedge h_{L^1} \wedge h_{H^2}\Big] \\
\wedge
\Big[
(h_{H^1} \wedge h_{1,lead})
\vee
(h_{L^2} \wedge h_{2,lead})
\vee
(h_{1,lead} \wedge h_{2,lead})
\Big].
\end{IEEEeqnarray*}
The first bracketed $[\cdot]$ term
ensures the car satisfies speed limits
and stays within the road limits.
The second bracketed term ensures one of the following holds:
(1)
the car stays in the first lane
and a safe distance from the lead car in the first lane,
(2)
the car stays in the second lane
and a safe distance from the lead car in the second lane,
or
(3)
the car stays a safe distance from both lead cars but can change lanes.

\section{Experiments}

\label{sec_experiments}

\subsection{Environments and Algorithms}

\label{sec_environments_and_algorithms}

In Sections \ref{sec_examples_fw}
and \ref{sec_examples_car}, we derived barrier functions for
fixed wing aircraft and self driving cars.
Although the direct computation of \eqref{eq_prgm}
is infeasible for non-convex systems, we consider
three efficient alternatives. First,
we use a Lagrangian computation, called BF LAG, which
has been employed previously in constrained
deep learning problems (e.g. \cite{ma2021feasible,yang2023model}).
In this approach, let $u_\phi:\mathbb{R}^{n_s}\times \mathbb{R}^{n_u}\to U$ be
a neural network parameterized by $\phi$.
Given a state and nominal control input pair,
the Lagrangian is
\begin{equation}
\mathcal{L}(s_k,\hat{u}_k) = \frac{1}{2}\norm{\hat{u}_k - u_\phi(s_k, \hat{u}_k)}^2 - \kappa c_h(s_k, u_\phi(s_k, \hat{u}_k))
\label{eq_lagrangian}
\end{equation}
where $\kappa\ge 0$ is a Lagrange multiplier.
We update $\phi$ by minimizing \eqref{eq_lagrangian} via gradient
descent.
Similarly, given a cost, we update $\kappa$ via 
$\kappa = \max(0, \kappa - rc_h(s_k,u_\phi(s_k)))$
where $r > 0$ is a learning rate (we set $r=1000$ in our experiments).
We additionally take the minimum of  $c_h$ and 0.0001 in \eqref{eq_lagrangian}
so that when minimizing the expectation of the Lagrangian
in \eqref{eq_lagrangian}, the average constraint
is less likely to be positive when
$u_\phi$ has some constraint violations in the dataset.
Because there is no guarantee that $u_\phi$ outputs
a safe action, we also include a line search
for the lowest cost, safe action $u_{l}$
between $\hat{u}_k$ and $\U_k$, the known safe action.
See \eqref{eq_fw_u_safe} and \eqref{eq_car_u_k}
for $\U_k$ given fixed-wing and car systems, respectively.
Further, even if $u_\phi$ outputs a safe action,
there is no guarantee that
the action is optimal
with respect
to \eqref{eq_prgm},
so we include a line search
between $\hat{u}_k$ and $u_\phi(s_k,\hat{u}_k)$
to find the lowest cost but safe action $u_{\phi,l}$ along the line
between $\hat{u}_k$ and $u_\phi(s_k,\hat{u}_k)$.
If both $\hat{u}_k$ and $u_\phi(s_k,\hat{u}_k)$
are unsafe, $u_{\phi,l}$ can be unsafe, so in this
case we apply $u_l$. Otherwise
we apply $u_{\phi,l}$ or $u_{l}$, whichever has the lowest
cost for \eqref{eq_prgm}.
The two other alternative overrides we evaluate are $u_{l}$
(called BF Line)
and $\U_k$ (called BF SINGLE).

For RL baselines, 
we compare a barrier override to algorithms discussed
for the GUARD Safe RL benchmark
\cite{zhao2024guard}. In particular, the baselines consist
of Trust Region Policy Optimization (TRPO) \cite{schulman2015trust},
Unrolling Safety Layer (USL) \cite{zhang2022saferl},
Constrained Policy Optimization (CPO) \cite{achiam2017constrained},
Feasible Actor Critic (FAC) \cite{ma2021feasible},
and LAG (TRPO with a Langrian cost term).
TRPO is the unconstrained baseline
to measure the peak performance that can be achieved without
considering safety. USL does an online action correction by
passing the action through a cost network and applying back propagation
to update the action until the cost is below a threshold.
CPO approximates a trust region constrained optimization
of surrogate functions with guaranteed improvement bounds.
LAG applies a policy loss
that includes a cost correction with a coefficient
that increases when the cost exceeds a threshold.
Finally, FAC extends the Lagrangian approach by making the
cost correction coefficient dependent on the state.

We evaluate the barrier function overrides against these baselines in two environments
for which we have derived barrier functions.
For every state that violates the safety condition,
the agent gets a cost of $1$.
Similarly, an episode is considered unsafe if any state
during the episode is unsafe.
To account for floating point error, we only count a state
as violating a constraint if it is more than $10^{-6}$ beyond a given
safety limit.
We include the number of unsafe episodes experienced during
training as it provides an interpretable cost budget
over the course of training, similar e.g. \cite{thananjeyan2021recovery}.

In our first environment,
we evaluate whether the barrier function in \eqref{eq_bf_fw}
can allow fixed wing aircraft
to stay within a safe flight envelope while using RL to learn to do waypoint following
(see Fig. \ref{fig_fw_screenshot}).
We use aircraft parameters from
\cite{boskovic2004adaptive,el2023design} and
choose reasonable values for $\mu_{max}, v_{min}, v_{max}, $ and $z_{min}$
(see Table \ref{tbl_fw_params}).
For the waypoint objective,
we start the aircraft at 
$x_0 = (17.5, 0, 0, 0, 0, 500)^T$ (velocity is meters per second,
angles are in radians, and positions are in meters)
and have 5 waypoints $w_{i}\in \mathbb{R}^3$
located at $w_{i+1} = w_i + (100, U_{i,y}, U_{i,z})^T$
where $w_0 = x_0$ and $U_{i,y},U_{i,z}$ are samples from a uniformly
distributed random variable over the range of $[-25, 25]$ meters
for $i=1,\ldots,5$. The observation consists of
scaled velocity 
(i.e. $(v_k - v_{min}) / (v_{min} - v_{max})$),
scaled $\gamma_k$
(i.e. $\gamma_k / \gamma_{max}$),
$\sin \psi_k$, $\cos \psi_k$,
the vector to the next waypoint relative to the current position
(i.e. $(w_j - (x_{k}, y_{k}, z_{k})^T) / O$ where $j$ is the index
of the next waypoint and $O = 50$ is a scaling factor),
and the scaled change to each subsequent waypoint
(i.e. $(w_{l+1} - w_{l}) / O$ for $l = j + 1,\ldots,6$).
Because the observation is a fixed size, we also append a valid
flag and fill the waypoint vectors with zeros when $l + 1 > 5$.
At every timestep, the agent receives a reward of 0.01 times the change
in distance to the next waypoint. A waypoint $w_i$ is reached when
$x_{k}$ is greater than the $x$-component of $w_i$. At this point
$e^{-d_{i,k} / 25}$ is added to the reward where $d_{i,k}$ is the
distance of the vehicle to $w_i$.
We set $\delta = 0.1$ and end an episode after 1000 timesteps,
the last waypoint is reached, it has taken more than 10 seconds
to reach the next waypoint, or the vehicle hits the ground.

Our second environment focuses on lane merging with adaptive
cruise control for self-driving cars.
In this scenario, the car under control
is in a two lane road with lead cars in both of the lanes.
The lead car initial offset is chosen from a uniform
distribution between 100 and 500 meters with initial speed
between 0 and 70 miles per hour. If an initial lead
car state results
in a system state $s_0$ where in $h_{car}(s_0) < 0$
then we sample again until a safe initial condition is found.
The lead car behavior continues at this speed for a uniformly
sampled time varying between 0 and 5 seconds at which point
it selects a new target speed between 0 and 70 miles per hour,
applying maximum acceleration or deceleration to achieve this speed.
The autonomous vehicle must learn to keep a safe distance behind
the lead car so that even in the rare case that the lead car decelerates
maximally to a stop, the autonomous vehicle can still avoid a collision
by similarly braking or by changing lanes, provided the vehicle in the other
lane is far enough ahead to allow this to be done safely.
When the autonomous vehicle passes a lead vehicle, we spawn a new
lead vehicle in the same lane using the same logic as the beginning of the episode.
The autonomous vehicle has a speed goal of $v_{tgt} = 70$ miles per hour which is the
same as the speed limit. It starts with an initial speed of 95\% of the
speed target.
The observation consists of the vector
\begin{equation*}
obs(s_k) =
\begin{bmatrix}
(d_{1,k} - 500) / 500 \\
(d_{2,k} - 500) / 500 \\
(y_{3,k} - W_{lane}) / W_{lane} \\
(v_{1,k} - v_{tgt}) / v_{tgt} \\
(v_{2,k} - v_{tgt}) / v_{tgt} \\
(v_{3,k} - v_{tgt}) / v_{tgt} \\
\sin \psi_{3,k} \\
\cos \psi_{3,k}
\end{bmatrix}.
\end{equation*}
where $d_{j,k} = x_{j,k} - x_{3,k}$ for $j=1,2$.
The reward is $r_k = 1 - |v_{3,k} - v_{tgt}| / v_{tgt}$
and the episode ends when the distance to the lead vehicle is less than
$D_{lead}$, the vehicle has exited the road, or 1000 timesteps have elapsed.
We use car parameters from
\cite{polack2017kinematic, ames2016control}
and choose reasonable values for $D_{lead}$, $v_{lim}$, $W_{lane}$,
and $W_{car}$ (see Table \ref{tbl_car_params}).

\subsection{Results Relative to Safe RL Baselines}

Results for the UAV environment
are shown in Figure \ref{fig_fw_results}
which plots
performance of the barrier function
override using $u_\phi$ (called BF Lag)
compared to the RL baselines.
As expected, TRPO has the highest performance
but as discussed in the next paragraph, also has the highest
cost.
USL has similar initial performance to TRPO because it has
a burn-in period for the first two million steps where the safety override is inactive.
However, once the safety override activates, performance
deteriorates without improving safety.
In two of the cases (LAG, CPO), the baseline cannot solve
the task while the final baseline, FAC, learns to reduce
unsafe behavior while producing significantly lower
reward than TRPO. 

Similarly, for cost, TRPO maintains a high cost throughout training. This is
partly because the waypoints can be placed outside of what is feasible for the
flight envelope but also that exceeding the speed limit increases the maximum
discounted reward. Comparing reward and cost, there is a clear
tradeoff amongst the baselines as algorithms that achieve lower cost
also achieve less reward. The only baselines to approximate 
zero cost (LAG and CPO) do so only near the end of training and do not solve
the task.
An interesting comparison is cost to the number of unsafe episodes.
Notably, all baselines encounter many unafe episodes,
meaning that there is at least one state that is unsafe.
This occurs because although the safety aware baselines make an attempt
to have lower cost, they still have safety violations so their employment
for safety critical applications is limited.

On the other hand, BF LAG achieves both high reward and zero cost during
training where the reward is comparable to the unconstrained baseline
and the cost is significantly improved over any of the baselines.
In particular, BF LAG allows the RL system to achieve
a higher reward than any of the constrained baselines while having
a comparable reward to the unconstrained baseline TRPO.
Notably, every baseline has at least 10,000 unsafe episodes whereas BF LAG
has zero. It is likely that BF LAG achieves less reward per episode than
the unconstrained baseline because we purposefully place waypoints
outside of what is achievable within the safe flight envelope. For instance,
given a max pitch of 10 degrees and 100 meters between waypoints,
the maximum altitude change is given by $100 \sin(\psi_{max})\approx 17$ meters,
which is less than the limits of how much altitude difference there
can be between waypoints of $25$ meters.

\begin{figure}
\centering

\begin{subfigure}{\columnwidth}
    \includegraphics[trim={0 65 0 0},clip]{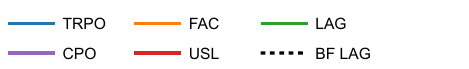}
\end{subfigure}
\begin{subfigure}{\columnwidth}
    \includegraphics{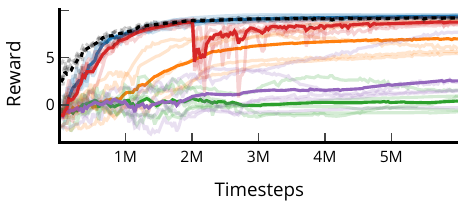}
\end{subfigure}
\begin{subfigure}{\columnwidth}
    \includegraphics{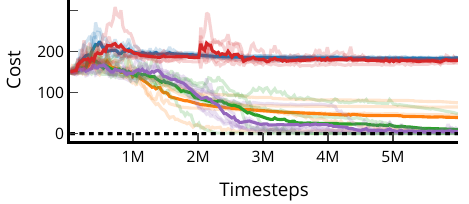}
\end{subfigure}
\begin{subfigure}{\columnwidth}
    \includegraphics{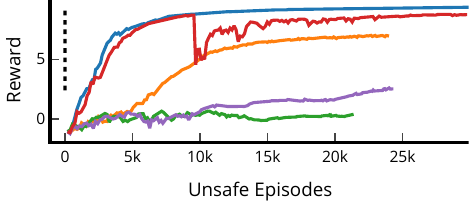}
\end{subfigure}
\caption{
Results in the waypoint following environment for a fixed wing UAV.
(top) Reward as a function of episode sample,
(middle) Cost as a function of episode sample,
(bottom) Reward vs cumulative number of unsafe episodes.
In the top and middle plots, faded lines
show the outcomes of the five repeated experiments
for a given algorithm
while non-faded lines
represent the mean of these experiments.
}
\label{fig_fw_results}
\end{figure}

\begin{table}
\centering
\caption{Fixed Wing UAV parameters}
\label{tbl_fw_params}
\begin{tabular}{l|r|l|l}
Parameter & Value & Units & Description \\\hline
$\rho$ & 1.2251  & $kg / m^3$ & air density \\
$W$ & 68.68  & $N$ & weight \\
$S$ & 1.058  & $m^2$ & wing surface area \\
$T_{max}$ & 20.60 & $N$ & maximum thrust \\
$[n_{min}, n_{max}]$ & $[-1.0, 2.5]$  & unitless & min/max lift factor \\
$C_{d0}$ & 0.02544  & unitless & parasitic drag coefficient \\
$k$ & 0.059 & unitless & induced drag coefficient \\
$\mu_{max}$ & 30 & deg & max bank angle \\
$[v_{min}, v_{max}]$ & $[15, 20]$ & $m / s$ & min/max speed \\
$z_{min}$ & 400 & $m$ & altitude floor \\
$\lambda$ & 0.5 & unitless & see \eqref{eq_bf_constraint}
\end{tabular}
\end{table}

In the car environment,
an episode can be quickly ended by running off the road.
This leads to non-intuitive results as
the unconstrained baseline has high cost but also the least
number of unsafe episodes of all the baselines.
Because unsafe behavior means that the episode ends and
the future rewards goes to zero, the unconstrained baseline
learns relatively safe behavior in terms of catastrophic
outputs (e.g. the lead car distance is less than $D_{lead}$
or running off the road). At the same time, it ignores
unsafe behavior that does not end the episode (going
faster than the speed limit, getting closer than
the time to collision limit). In other words, it learns
to avoid behaviors that lead to catastrophic outcomes
while allowing ones that have no effect on the reward.
Other safe baselines appear to have the opposite
approach. Because they are trained to avoid high cost,
they learn to create catastrophic outcomes because those
scenarios lead to a cost of 1 rather than 
keeping the episode going and inducing high costs
from keeping the episode going.
Results are in Figure \ref{fig_car_results}.
On the other hand, BF LAG does not suffer from this tuning
challenge as it achieves high reward and zero cost.
Its initial high performance is due to it being able to
keep episodes going the full 1000 timesteps.
After this it learns to effectively pass cars to achieve
speeds close to the target speed without having a collision.

\begin{figure}
\centering

\begin{subfigure}{\columnwidth}
    \includegraphics[trim={0 65 0 0},clip]{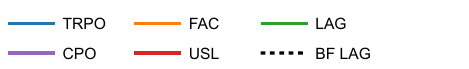}
\end{subfigure}
\begin{subfigure}{\columnwidth}
    \includegraphics{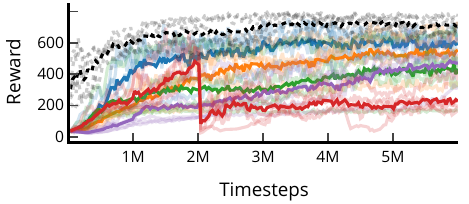}
\end{subfigure}
\begin{subfigure}{\columnwidth}
    \includegraphics{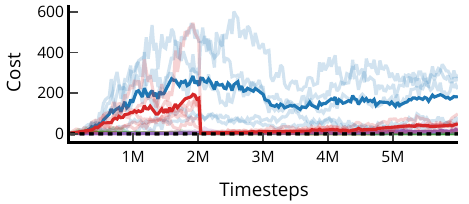}
\end{subfigure}
\begin{subfigure}{\columnwidth}
    \includegraphics{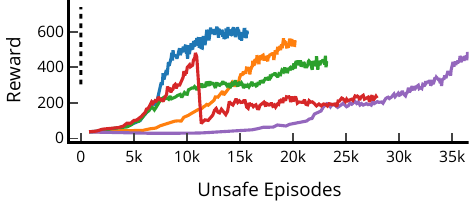}
\end{subfigure}
\caption{
Results in car environment.
(top) Reward as a function of episode sample,
(middle) Cost as a function of episode sample,
(bottom) Reward vs cumulative number of unsafe episodes.
In the top and middle plots, faded lines
show the outcomes of the five repeated experiments
for a given algorithm
while non-faded lines
represent the mean of these experiments.
}
\label{fig_car_results}
\end{figure}

\begin{table}
\centering
\caption{Car parameters}
\label{tbl_car_params}
\begin{tabular}{l|r|l|l}
Parameter & Value & Units & Description \\\hline
$l_f$ & 1.17 & $m$ & front distance \\
$l_r$ & 1.77 & $m$ & rear distance \\
$[u_{1,min},u_{1,max}]$ & $[-2.87, 2.87]$ & $m$ & min/max lead \\
&&& acceleration \\
$\A_{3}$ & $[2.86, 2.86]$ & $m$ & min/max \\
&&& acceleration \\
$[u_{2,min},u_{2,max}]$ & $[-1, 1]$ & deg & min/max steering \\
&&& angle \\
$\tau$ & $1.8$ & sec & time to \\
&&& collision \\
$W_{car}$ & 1.83 & $m$ & car width \\
$W_{lane}$ & 3.6 & $m$ & lane width \\
$\lambda$ & 0.5 & unitless & see \eqref{eq_bf_constraint}
\end{tabular}
\end{table}

\subsection{Results of Barrier Override Alternatives}

While BF LAG (the controller $u_\phi$ described
in Section \ref{sec_environments_and_algorithms})
is theoretically motivated by the 
Karush-Kuhn-Tucker conditions \cite{boyd2004convex},
it nevertheless has some disadvantages. First,
choosing the initial value for the Lagrange multiplier $\kappa$ and the associated
learning
rate can be difficult as it is affected by differences
in scaling for the objective and constraint.
This can be solved by hyperparameter tuning but
this is unlikely to generalize across environments
or even while training in a single environment
\cite{haarnoja2018soft}.
For instance, in \cite{zhao2024guard}, the initial
value and learning rate for the Lagrange multiplier is 0
and 0.005, respectively, for FAC whereas we find an initial
value and learning rate of 10 and 1000 work better for
$u_\phi$.
If the multiplier is too small,
then $u_\phi$ may train too slowly to learn safe actions.
On the other hand, if it is too large, $u_\phi$ may be
conservative and not approximate $\hat{u}_k$ particularly well.
A second challenge of the Lagrange approach is
that it violates the Markov assumption when training
the RL policy which is a theoretical challenge
for RL (discussed e.g. in \cite{foerster2017stabilising}). To see
how the Markov assumption is violated, note that when there is an override,
the system has dynamics $f(s_k, u_\phi(s_k))$ and since
$u_\phi(s_k)$ is changing due to updates from minimizing \eqref{eq_lagrangian},
the dynamics are changing as well.
Third, training $u_\phi$ will slow down learning as the additional
backpropagation adds computational overhead.
Fourth,
there is no guarantee that $u_\phi$ will
output a safe action even if $s_k$ is such that $h(s_k) \ge 0$
(see Figure \ref{fig_nonconvex_ovr_unsafe_pct}).
Similarly, if $u_\phi$  outputs a safe action,
there is no guarantee that the output is the closest
possible $u_k\in U$ to $\hat{u}_k$ such that $c_h(s_k, u_\phi(s_k, \hat{u}_k)) \ge 0$.
This necessitates the line search discussed above for BF LAG.
Similar to the third issue,
the line search implies a higher computational overhead
as it involves additional constraint computations which
can be computationally intensive when either the dynamics
in \eqref{eq_dyn} or the barrier function have a significant
computational burden.
Finally, since $u_\phi$ is a neural network, interpreting
its output is difficult. For a safety critical
application, particularly where humans are involved,
having a predictable override can increase interpretability
and explainability as these factors are critical for trust
\cite{mehrotra2024systematic}. Consider for instance a self-driving car
example. A known override
that
the car will decelerate at a known rate and steer so that
the car is parallel to the road is more interpretable
than an override that is a neural network.

Figures \ref{fig_fw_results} and \ref{fig_car_results}
indicate that barrier functions can enable safe control 
overrides even when time is discrete or the dynamics are
non-convex in the control input. We now compare
the theoretically motivated override $u_\phi$
with more computationally efficient and more easily
tuned overrides $u_l$ (BF Line) and $\widetilde{u}$ (BF Single) 
in Figure
\ref{fig_ovr_fw_results}.
Surprisingly, the performance of BF LAG is nearly
the same as BF Line and BF Single for the fixed wing environment
(top plot of
Figure \ref{fig_ovr_fw_results})
and in either environment, BF Single has similar performance
to the top performing baselines.
This is the case even though the distance between the override and
the nominal control value is significantly smaller
for BF Lag than BF Single (top middle plot of Figure \ref{fig_ovr_fw_results}).
Given the relatively large distance
of the override to the nominal,
one would expect that, for all else equal,
the deployed action would have a high standard deviation,
since it may be switching between the nominal and conservative
override of $\widetilde{u}_k$. From the bottom middle plot
of Figure \ref{fig_ovr_fw_results} though, this is not the case,
as BF Single actually has the lowest action standard deviation.
Further, in the bottom plot of Figure \ref{fig_ovr_fw_results},
BF Single also has the least percentage of overrides of the alternatives.
We posit from this that what is actually happening is that
the RL system has learned that the BF Single override is significant
and adapts the policy so that it avoids having an override take place.
In other words, for the fixed wing environment, having an override
that better approximates \eqref{eq_prgm} does not yield much benefit
because the RL system compensates for suboptimal solutions
to \eqref{eq_prgm}.

On the other hand,
Figure \ref{fig_ovr_car_results}
shows that BF Lag does have benefits over BF Single in the car environment.
However, even in this case, BF Single
still matches the performance of the highest
performing baseline (TRPO). Noting
the significant number of violations for any of the RL
baselines, it appears that being able to
specify a barrier function for a system
is sufficient for high performance safety assurance
for both the fixed wing and car environments.
An optimal computation of \eqref{eq_prgm} is not
necessary to beat cost aware RL even for 
discrete time, non-convex systems for the evaluated
environments.

\begin{figure}
\centering

\begin{subfigure}{\columnwidth}
    \includegraphics[trim={0 65 0 0},clip]{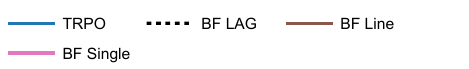}
\end{subfigure}
\begin{subfigure}{\columnwidth}
    \includegraphics{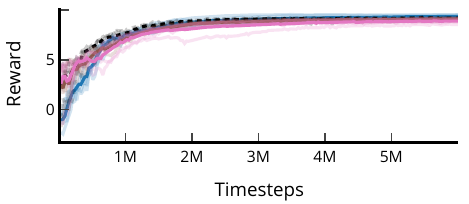}
\end{subfigure}
\begin{subfigure}{\columnwidth}
    \includegraphics{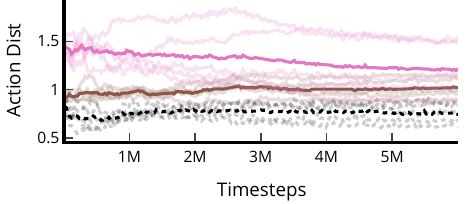}
\end{subfigure}
\begin{subfigure}{\columnwidth}
    \includegraphics{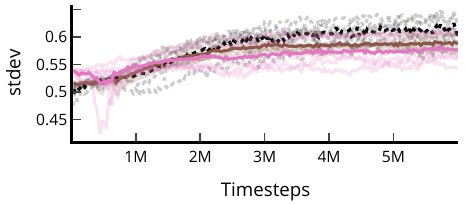}
\end{subfigure}
\begin{subfigure}{\columnwidth}
    \includegraphics{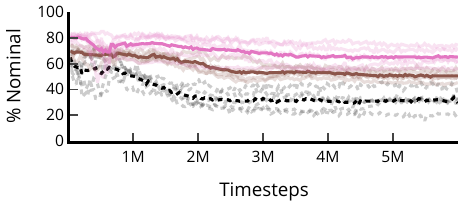}
\end{subfigure}
\caption{
A comparison of approximate barrier function overrides
in the fixed wing UAV environment.
(top) Reward as a function of episode sample,
(top middle) Average distance of the override to the nominal control value,
(bottom middle) Standard deviation of the control value.
(bottom) Percentage of actions where the nominal override is used
rather than an override.
Faded lines
show the outcomes of the five repeated experiments
for a given algorithm
while non-faded lines
represent the mean of these experiments.
}
\label{fig_ovr_fw_results}
\end{figure}

\begin{figure}
\centering

\begin{subfigure}{\columnwidth}
    \includegraphics[trim={0 65 0 0},clip]{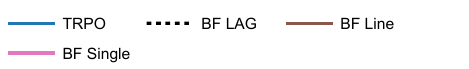}
\end{subfigure}
\begin{subfigure}{\columnwidth}
    \includegraphics{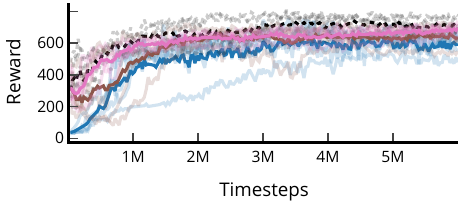}
\end{subfigure}
\begin{subfigure}{\columnwidth}
    \includegraphics{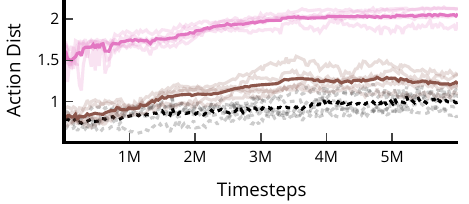}
\end{subfigure}
\begin{subfigure}{\columnwidth}
    \includegraphics{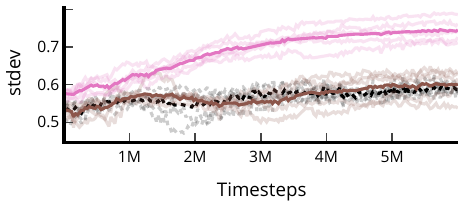}
\end{subfigure}
\begin{subfigure}{\columnwidth}
    \includegraphics{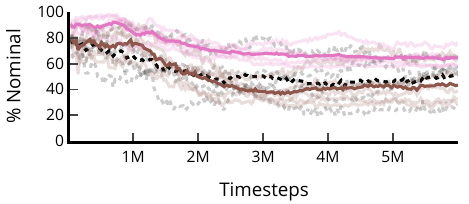}
\end{subfigure}
\caption{
A comparison of approximate barrier function overrides
in the car environment.
(top) Reward as a function of episode sample,
(top middle) Average distance of the override to the nominal control value,
(bottom middle) Standard deviation of the control value.
(bottom) Percentage of actions where the nominal override is used
rather than an override.
Faded lines
show the outcomes of the five repeated experiments
for a given algorithm
while non-faded lines
represent the mean of these experiments.
}
\label{fig_ovr_car_results}
\end{figure}

This observation has implications for RL deployment as there
are significant downsides to employing BF Lag discussed above.
Here we focus on the computational burden that
BF Lag induces.
RL often has a set of hyperparameters
that must be tuned so avoiding also tuning
the Lagrangian initial value and learning rate is a significant
benefit of BF Single even if it results in somewhat less performance
in the car environment. In addition, there is a significant
computational overhead of BF Lag, primarily due to additional
computations of the constraint (see Figure \ref{fig_num_steps}).
In Table \ref{tbl_computation} we compare the compuation times
of BF Single and BF Lag, finding that
BF Single training times are 1.7 and 3.0 times faster than BF Lag
for the fixed-wing and car environments, respectively.
In both environments, the key difference is that BF Single does
less constraint calculations and this is magnified
in the car environment because the barrier function
is more computationally intensive. The reason for this is the multi-step horizon
in computing $h_{L^1}, h_{L^2}, h_{H^1}$, and $h_{H^2}$ which causes
$c_h$ to become a bottleneck both in the forward pass as well as the backward pass
of $\mathcal{L}$.
In particular, because $u_\phi$ has no assurance that it will output a
safe action (see Figure \ref{fig_nonconvex_ovr_unsafe_pct}),
we must include a line search
to ensure that a safe override always occurs.
Whether BF Lag has any beneficial effect on performance appears
to be environment specific. However, given the downsides of 
BF Lag and the strong performance of BF Single relative to RL
baselines, this makes BF Single a reasonable approach
for computing an override for non-convex problems
when an override is included in RL training so that the RL
system can incorporate the effect of the override in developing
a policy.

\begin{figure}
\centering

\includegraphics{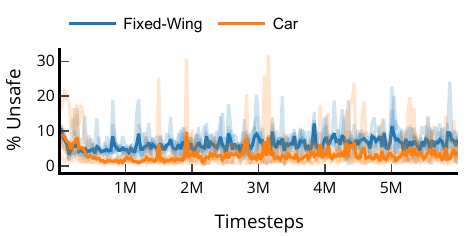}
\caption{%
The percentage of cases where $u_\phi$ (BF Lag)
outputs an unsafe action.
}
\label{fig_nonconvex_ovr_unsafe_pct}
\end{figure}

\begin{table}
\centering
\caption{
        Cumulative computation times in seconds per training epoch broken into
        data generation and training phases.
        Training is done on a Intel(R) Core(TM) i7-8700 CPU (3.20GHz)
        cpu. Gpu acceleration did not speed up training because
        the neural networks are relatively small.
    }
\label{tbl_computation}
\begin{tabular}{l  l   c  c | c  c}
\multicolumn{2}{c}{} &  \multicolumn{2}{c}{FW} & \multicolumn{2}{c}{Car} \\ \cline{3-6}\cline{3-6}
\multicolumn{2}{c}{} & \multicolumn{1}{c}{BF Lag} & \multicolumn{1}{c|}{BF Single} & \multicolumn{1}{c}{BF Lag} & \multicolumn{1}{c}{BF Single} \\\cline{3-6}
\multirow{3}{*}{\rotatebox[origin=c]{90}{\underline{Gen Data} \quad}} & $\hat{u}_k$  & 9.3 & 9.1 & 9.3 & 9.1\\
 & $c_h$  & 10.2 & 1.1 & 33.3 & 4.4\\
 & Env Step  & 4.6 & 4.5 & 3.9 & 3.9\\
 & $h$  & 0.5 & 0.5 & 3.4 & 3.3\\
 & $u_\phi$  & 3.2 &  & 2.3 & \\
 & $\widetilde{u}_k$  & 0.5 & 0.1 & 0.7 & 0.3\\
\hline
\multirow{3}{*}{\rotatebox[origin=c]{90}{\underline{Train} \qquad}} & $\mathcal{L}$ back  & 0.9 &  & 11.9 & \\
 & $c_h$ fwd  & 0.2 &  & 7.1 & \\
 & Nominal  & 4.7 & 4.6 & 4.6 & 4.5\\
 & $u_\phi$ fwd  & 0.4 &  & 0.4 & \\
 & $h$ fwd  & 0.0 &  & 0.1 & \\
 & $\mathcal{L}$ fwd  & 0.0 &  & 0.1 & \\\hline \hline
 &  Gen Data Total & 28.2 & 15.3 & 52.9 & 21.1\\
 &  Train Total & 6.3 & 4.6 & 24.2 & 4.5\\
 &  Total & 34.5 & 19.9 & 77.1 & 25.5
\end{tabular}
\end{table}
\section{Conclusion}

In this paper we have discussed
issues with applying barrier functions
to RL.
In particular, although RL can be applied
to dynamics that are non-convex in the control
input, this can make computing an override with
a barrier function computationally intractable.
After deriving barrier functions for two systems
that are non-convex in the control input, namely
fixed wing aircraft and self-driving cars,
we demonstrated that approximations of an optimal
override are a computationally efficient way to
ensure safety while preserving the performance
capabilities of an RL system. In our comparison
to safe RL baselines for fixed wing aircraft and
self-driving cars, a barrier function with an approximate
safe override has zero safety violations while
having comparable or better performance relative
to model free RL approaches.
This suggests that barrier functions can be applied
to non-convex systems while retaining the performance
benefits of an RL system.

\label{sec_conclusion}

\section*{Acknowledgements}

The authors would like to thank the Georgia Tech Research Institute for funding
this research out of Independent Research and Development (IRAD) funds.
Additionally, the authors would like to thank Dr. Klimka Kasraie for her valuable
feedback on this paper.

\addtolength{\textheight}{-1.5cm} 
\vspace{-0.5em}
\bibliographystyle{IEEEtran}
\bibliography{inputs/nonconvex}

\end{document}